\documentclass{article}

\usepackage[nonatbib, final]{neurips_2020}

\usepackage[utf8]{inputenc} 
\usepackage[T1]{fontenc}    
\usepackage{hyperref}       
\usepackage{url}            
\usepackage{booktabs}       
\usepackage{amsfonts}       
\usepackage{nicefrac}       
\usepackage{microtype}      
\usepackage{bm}
\usepackage{amsmath, amssymb}
\usepackage{amsfonts}
\usepackage{amsthm}
\usepackage{xcolor}
\usepackage{stmaryrd}
\usepackage{caption}
\usepackage{subcaption}
\usepackage{graphicx}
\usepackage{bbm}
\usepackage{wrapfig}
\usepackage[normalem]{ulem}
\usepackage{cite}
\usepackage{textcomp}
\usepackage{algpseudocode}
\usepackage{algorithm}

\newtheorem{theorem}{Theorem}
\newtheorem{lemma}{Lemma}
\newtheorem*{lemma*}{Lemma}
\newtheorem*{theorem*}{Theorem}
\newtheorem{definition}{Definition}

\newtheorem*{proposition*}{Proposition}

\newtheorem*{conjecture*}{Conjecture}
\newtheorem{example}{Example}
\newtheorem*{example*}{Example}

\newcommand{\R}{\mathbb{R}}

\title{Neural Anisotropy Directions}

\author{%
Guillermo Ortiz-Jiménez\thanks{Equal contribution. Correspondence to \texttt{\{guillermo.ortizjimenez, apostolos.modas\}@epfl.ch}. The code to reproduce our experiments can be found at \href{https://github.com/LTS4/neural-anisotropy-directions}{https://github.com/LTS4/neural-anisotropy-directions}.}\\
EPFL, Lausanne, Switzerland\\
\texttt{guillermo.ortizjimenez@epfl.ch}
\And
Apostolos Modas\footnotemark[1]\\
EPFL, Lausanne, Switzerland\\
\texttt{apostolos.modas@epfl.ch}
\And
Seyed-Mohsen Moosavi-Dezfooli\\
ETH Zürich, Zurich, Switzerland\\
\texttt{seyed.moosavi@inf.ethz.ch}
\And
Pascal Frossard\\
EPFL, Lausanne, Switzerland\\
\texttt{pascal.frossard@epfl.ch}
}

\begin{document}
\maketitle
\setcounter{footnote}{0}

\begin{abstract}
In this work, we analyze the role of the network architecture in shaping the inductive bias of deep classifiers. To that end, we start by focusing on a very simple problem, i.e., classifying a class of linearly separable distributions, and show that, depending on the direction of the discriminative feature of the distribution, many state-of-the-art deep convolutional neural networks (CNNs) have a surprisingly hard time solving this simple task. We then define as \emph{neural anisotropy directions (NADs)} the vectors that encapsulate the directional inductive bias of an architecture. These vectors, which are specific for each architecture and hence act as a signature, encode the preference of a network to separate the input data based on some particular features. We provide an efficient method to identify NADs for several CNN architectures and thus reveal their directional inductive biases. Furthermore, we show that, for the CIFAR-10 dataset, NADs characterize the features used by CNNs to discriminate between different classes.
\end{abstract}

\section{Introduction}

In machine learning, given a finite set of samples, there are usually multiple solutions that can perfectly fit the training data, but the \emph{inductive bias} of a learning algorithm selects and prioritizes those solutions that agree with its \emph{a priori} assumptions~\cite{mitchellNeedBiasesLearning, battagliaRelationalInductiveBiases2018}. Arguably, the main success of deep learning has come from embedding the right inductive biases in the architectures, which allow them to excel at tasks such as classifying ImageNet~\cite{dengImageNetLargescaleHierarchical2009}, understanding natural language~\cite{wangGLUEMultiTaskBenchmark2018}, or playing Atari~\cite{mnihHumanlevelControlDeep2015}.

Nevertheless, most of these biases have been generally introduced based on heuristics that rely on generalization performance to \emph{naturally select} certain architectural components. As a result, although these deep networks work well in practice, we still lack a proper characterization and a full understanding of their actual inductive biases. In order to extend the application of deep learning to new domains, it is crucial to develop generic methodologies to systematically identify and manipulate the inductive bias of deep architectures.

Towards designing architectures with desired properties, we need to better understand the bias of the current networks. However, due to the co-existence of multiple types of inductive biases within a neural network, such as the preference for simple functions~\cite{rahamanSpectralBiasNeural2019}, or the invariance to certain group transformations~\cite{mallatUnderstandingDeepConvolutional2016}, identifying all biases at once can be challenging. In this work, we take a bottom-up stance and focus on a fundamental bias that arises in deep architectures even for classifying linearly separable datasets. In particular, we show that depending on the nature of the dataset, some deep neural networks can only perform well when the discriminative information of the data is aligned with certain directions of the input space. We call this bias the \emph{directional inductive bias} of an architecture. 

This is illustrated in Fig.~\ref{fig:dfts} for state-of-the-art CNNs classifying a set of linearly separable distributions with a single discriminative feature lying in the direction of some Fourier basis vector\footnote{The exact settings of this experiment will be described in Sec.~\ref{sec:spectral_bias}. In general, all training and evaluation setups, hyperparameters, number of training samples, and network performances are listed in the Supp. material.}. Remarkably, even the gigantic DenseNet~\cite{huangDenselyConnectedConvolutional2017} only generalizes to a few of these distributions, despite common belief that, due to their superior capacity, such networks can learn most functions efficiently. Yet, even a simple logistic regression eclipses their performance on a simple linearly separable task.

In this paper, we aim to explain why this happens, and try to understand why some linear distributions are easier to classify than others. To that end, we introduce the concept of \emph{neural anisotropy directions} to characterize the directional inductive bias of an architecture.
\begin{definition}[Neural anisotropy directions]
The neural anisotropy directions (NADs) of a specific architecture are the ordered set of orthonormal vectors $\{\bm{u}_i\}_{i=1}^D$ ranked in terms of the preference of the network to separate the data in those particular directions of the input space.
\end{definition}

In general, though, quantifying the preference of a complex network to separate data in certain directions is not straightforward. In this paper, we will show that measuring the performance of a network on different versions of a linearly separable dataset can reveal its directional inductive bias. Yet, we will provide an efficient computational method to fully characterize this bias in terms of NADs, independent of training data. Finally, we will reveal that NADs allow a network to prioritize certain discriminating features of a dataset, and hence act as important conductors of generalization.

Our main contributions can be summarized as follows:
\begin{itemize}
    \item We characterize the directional inductive bias in the spectral domain of state-of-the-art CNNs, and explain how pooling layers are a major source for this bias.
    \item More generally, we introduce a new efficient method to identify the NADs of a given architecture using only information available at initialization.
    \item Finally, we show that the importance of NADs is not limited to linearly separable tasks, and that they determine the selection of discriminative features of CNNs trained on CIFAR-10.
\end{itemize}

We believe that our findings can impact future research in novel architectures, by allowing researchers to compare and quantify the specific inductive bias of different networks.

\begin{figure*}[t] 
\includegraphics[width=\textwidth]{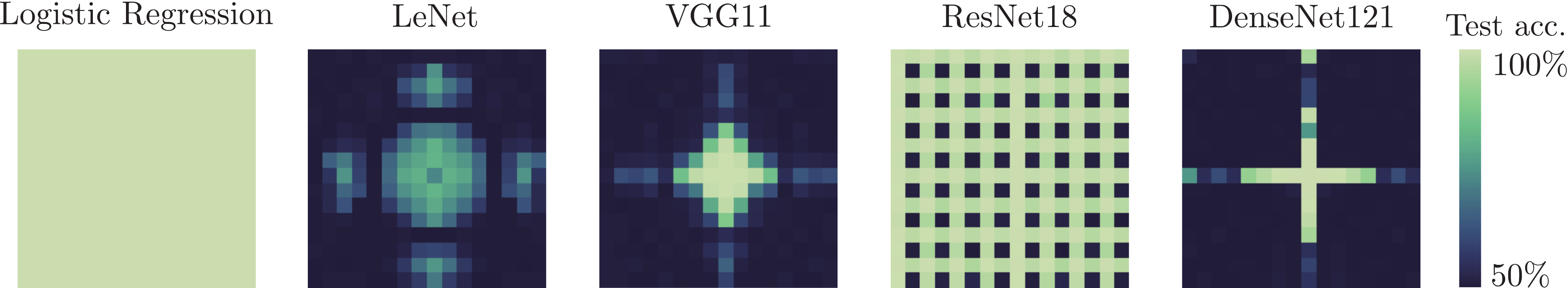}
\caption{Test accuracies of different architectures~\cite{lecunGradientbasedLearningApplied1998, simonyanVeryDeepConvolutional2015, resnet, huangDenselyConnectedConvolutional2017}. Each pixel corresponds to a linearly separable dataset (with $10,000$ training samples) with a single discriminative feature aligned with a basis element of the 2D-DFT. We use the standard 2D-DFT convention and place the dataset with lower discriminative frequencies at the center of the image, and the higher ones extending radially to the corners. All networks (except LeNet) achieve nearly $100\%$ train accuracy. ($\sigma=3$, $\epsilon=1$)}
\label{fig:dfts}

\end{figure*}

\paragraph{Related work}

The inductive bias of deep learning has been extensively studied in the past. From a theoretical point of view, this has mainly concerned the analysis of the implicit bias of gradient descent~\cite{zhangUnderstandingDeepLearning2017,soudryImplicitBiasGradient2018,gunasekarCharacterizingImplicitBias2018,rahamanSpectralBiasNeural2019,chaudhari2018stochastic}, the stability of convolutional networks to image transformations~\cite{mallatUnderstandingDeepConvolutional2016,biettiInductiveBiasNeural2019}, or the impossibility of learning certain combinatorial functions~\cite{nyeAreEfficientDeep2018,abbeProvableLimitationsDeep2019}. Related to our work, it has been shown that some heavily overparameterized neural networks will provably learn a linearly separable distribution when trained using stochastic gradient descent (SGD)~\cite{brutzkus2018sgd}. This result, however, only applies to the case of neural networks with two fully connected layers, and it says little about the learnability of linearly separable distributions with complex architectures. On the practical side, the myriad of works that propose new architectures are typically motivated by some informal intuition on their effect on the inductive bias of the network~\cite{lecunGradientbasedLearningApplied1998, simonyanVeryDeepConvolutional2015, resnet,huangDenselyConnectedConvolutional2017, vaswaniAttentionAllYou2017, defferrardConvolutionalNeuralNetworks, ioffeBatchNormalizationAccelerating2015}. Although little attention is generally given to properly quantifying these intuitions, some works have recently analyzed the role of architecture in the translation equivariance of modern CNNs~\cite{needle, zhangMakingConvolutionalNetworks2019, Cordonnier2020On}.

\section{Directional inductive bias}
\label{sec:spectral_bias}

We will first show that the test accuracy on different versions of a linearly separable distribution can reveal the directional inductive bias of a network towards specific directions. In this sense, let $\mathcal{D}(\bm{v})$ be a linearly separable distribution parameterized by a unit vector $\bm{v}\in\mathbb{S}^{D-1}$, such that any sample $(\bm{x}, y)\sim\mathcal{D}(\bm{v})$ satisfies $\bm{x}=\epsilon y\bm{v}+\bm{w}$, with noise $\bm{w}\sim\mathcal{N}\left(\bm{0}, \sigma^2(\bm{I}_D-\bm{v}\bm{v}^T)\right)$ orthogonal to the direction $\bm{v}$, and $y$ sampled from $\{-1, +1\}$ with equal probability. Despite $\mathcal{D}(\bm{v})$ being linearly separable based on $\bm{v}$, note that if $\epsilon \ll \sigma$ the noise will dominate the energy of the samples, making it hard for a classifier to identify the generalizing information in a finite-sample dataset.

In practice, it is not feasible to test the performance of a classifier on all possible versions of $\mathcal{D}(\bm{v})$. Nevertheless, one can at least choose a spanning basis of $\mathbb{R}^D$, from where a set of possible directions $\{\bm{v}_i\}_{i=1}^D$ can be picked. Informally speaking, if a direction is aligned with the inductive bias of the network under study, then its performance on $\mathcal{D}(\bm{v})$ would be very good. Otherwise, it would be bad.

We validate our hypothesis on common CNNs used for image classification with a $32\times 32$ single-channel input. We use the two-dimensional discrete Fourier basis (2D-DFT) -- which offers a good representation of the features in standard vision datasets~\cite{yinFourierPerspectiveModel2019,wangHighFrequencyComponent2020, ortiz-jimenezHoldMeTight2020} -- to generate the selected vectors\footnote{For the exact procedure and more experiments with similar findings see Sec.~A.1 and A.2 of Supp. material.}. The difference in performance on these experiments underlines the strong bias of these networks towards certain frequency directions (see Fig.~\ref{fig:dfts}). Surprisingly, beyond test accuracy, the bias can also be identified during training, as it takes much longer to converge for some data distributions than others, even when they have little noise (see Fig.~\ref{fig:real_optimization}). This is, the directional inductive bias also plays a role in optimization.

\begin{figure}[ht!]
\begin{minipage}[t]{.485\textwidth}
  \centering
\includegraphics[width=\textwidth]{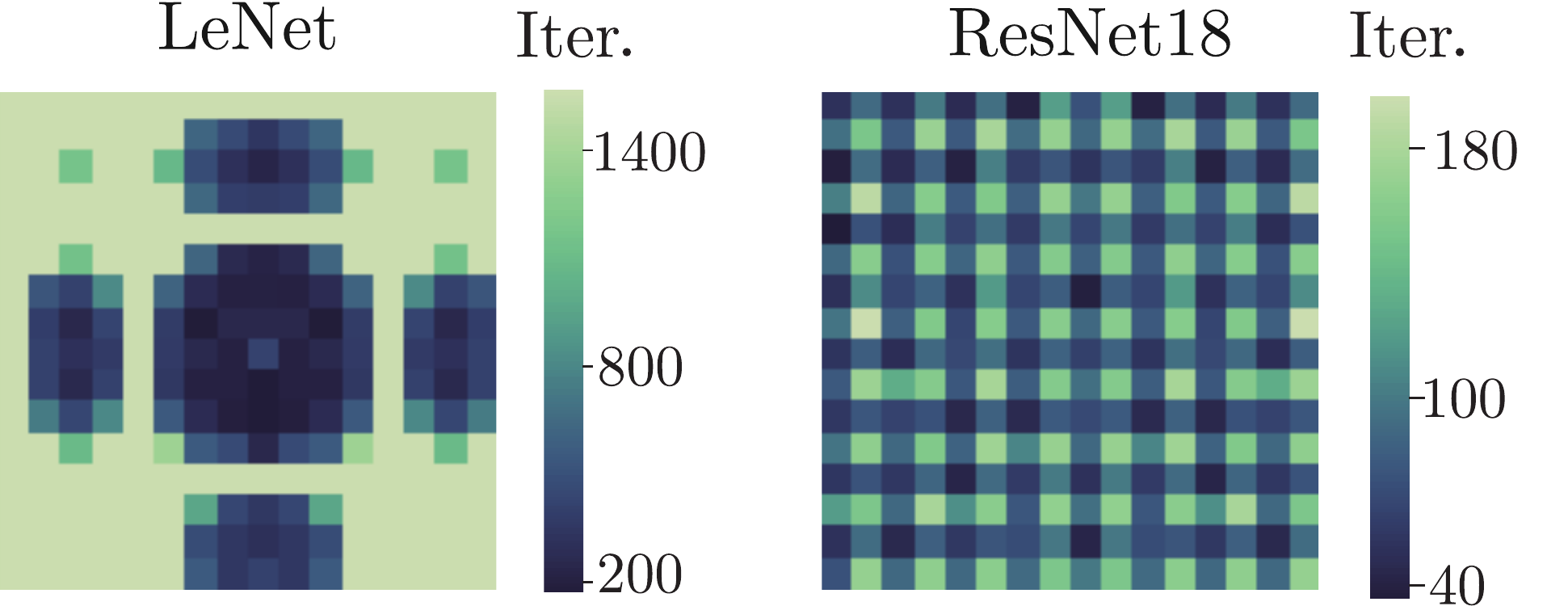}
\caption{Training iterations required to achieve a small training loss on different $\mathcal{D}(\bm{v}_i)$ aligned with some Fourier basis vectors ($\sigma=0.5$). }
\label{fig:real_optimization}
\end{minipage}
\hfill
\begin{minipage}[t]{.45\textwidth}
  \centering
\includegraphics[width=\textwidth]{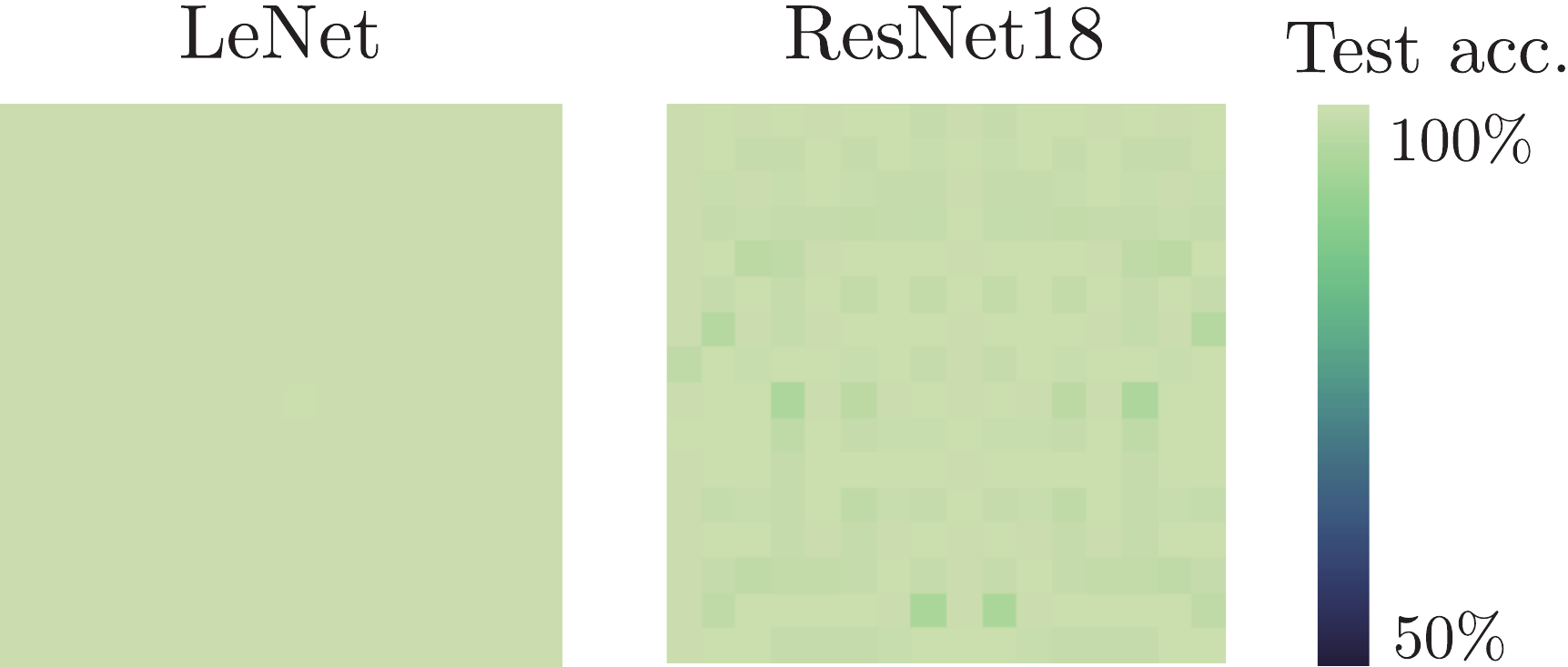}
\caption{Test accuracies on different $\mathcal{D}(\bm{v}_i)$ aligned with some Fourier basis vectors when removing all pooling layers.}
\label{fig:no_pooling}
\end{minipage}
\end{figure}

One plausible explanation for these results is that under multiple possible fitting solutions on a finite dataset, the network prioritizes those that are aligned with certain frequency directions. Therefore, if the prioritized solutions align with the discriminative features of the data, the classifier can easily generalize. Otherwise, it just overfits to the training data.
Finally, we can note the diverse patterns in Fig~\ref{fig:dfts} and Fig~\ref{fig:real_optimization}. CNNs are composed of many modules, and their interconnection can shape the inductive bias of the network in complex ways. In particular, as seen in Fig.~\ref{fig:no_pooling}, if we remove pooling from these networks (with fully connected layers properly adjusted) their performance on different frequencies is equalized. Pooling has previously been shown to modulate the inductive bias of CNNs in the spatial domain~\cite{zhangMakingConvolutionalNetworks2019}; however, it seems that it does so in the spectral domain, as well. This also confirms that the overfitting of these models on this na{\"i}ve distribution cannot simply be due to their high complexity, as removing pooling technically increases their capacity, and yet their test accuracy improves.

In general, a layer in the architecture can shape the bias in two main ways: by causing an anisotropic loss of information, or by anisotropically conditioning the optimization landscape. In what follows we describe each of them and illustrate their effect through the example of a linear pooling layer.

\subsection{Anisotropic loss of information}

We refer to an anisotropic loss of information as the result of any transformation that harms generalization in specific directions, e.g., by injecting noise with an anisotropic variance. Under these circumstances, any information in the noisy directions will not be visible to the network. 

Let $\hat{\bm{x}}=\mathcal{F}(\bm{x})$ denote the Fourier transform of an input vector $\bm{x}$ entering a linear pooling layer (e.g.,~average pooling) with a subsampling factor $S$. Without loss of generality, let $(\bm{x},y)\sim\mathcal{D}(\bm{v}_\ell)$, and $\bm{v}_{\ell}=\mathcal{F}^{-1}(\bm{e}_{\ell})$ with $\bm{e}_\ell$ representing the $\ell$-th canonical basis vector of $\mathbb{R}^D$. Then, the Fourier transform of the output of the pooling layer satisfies $\hat{\bm{z}}=\bm{A}(\hat{\bm{m}}\odot\hat{\bm{x}})$ where $\bm{A}\in\R^{M\times D}$ represents an aliasing matrix such that $\bm{A}=\frac{1}{\sqrt{S}}\begin{bmatrix}\bm{I}_M & \cdots & \bm{I}_M\end{bmatrix}$ with $M=\lceil D/S\rceil$. Here, $\hat{\bm{m}}\odot\hat{\bm{x}}$ is the representation in the spectral domain of the convolution of a prefilter $\hat{\bm{m}}$, e.g., average filtering, with the input signal. Expanding this expression, the spectral coefficients of the output of pooling become
\begin{equation}
    \hat{\bm{z}}[t]=\cfrac{1}{\sqrt{S}}\,\sum_{k=0}^{S-1}\hat{\bm{m}}\llbracket k\cdot M + t\rrbracket_D\;\hat{\bm{x}}\llbracket k\cdot M + t\rrbracket_D,\label{eq:aliasing}
\end{equation}
where $\hat{\bm{x}}\llbracket i \rrbracket_D$ represents the $(i \operatorname{mod} D)$-th entry of $\hat{\bm{x}}$. The following theorem expresses the best achievable performance of any classifier on the distribution of the output of pooling.

\begin{theorem}[Bayes optimal classification accuracy after pooling]\label{thm:bayes}
After pooling, the best achievable test classification accuracy on the distribution of samples drawn from $\mathcal{D}(\bm{v}_\ell)$ can be written as
\begin{equation}
    1-\mathcal{Q}\left(\cfrac{\sqrt{2}\epsilon}{2\sigma}\,\gamma(\ell)\right)\quad \text{with}\quad \gamma^2(\ell)=\cfrac{|\hat{\bm{m}}[\ell]|^2 \cdot S}{\sum_{k=1}^{S-1}\left|\hat{\bm{m}}\llbracket\ell+k\cdot M\rrbracket_D\right|^2},
\end{equation}
and $\mathcal{Q}(\cdot)$ representing the tail distribution function of the standard normal distribution. 
\end{theorem}
\begin{proof}
See Sec.~B.1 of Supp. material.
\end{proof}

The intuition behind this theorem lies in \eqref{eq:aliasing}. Note that after pooling the discriminative information appears only at position $\llbracket\ell\rrbracket_M$ and that its signal-to-noise ratio $\gamma(\ell)$ is completely characterized by $\ell$. For this reason, we say that pooling acts as an anisotropic lossy information channel (see~Fig~\ref{fig:bayes_freq}).

\begin{figure}[t!]
\begin{subfigure}{0.49\textwidth}
    \includegraphics[width=\textwidth]{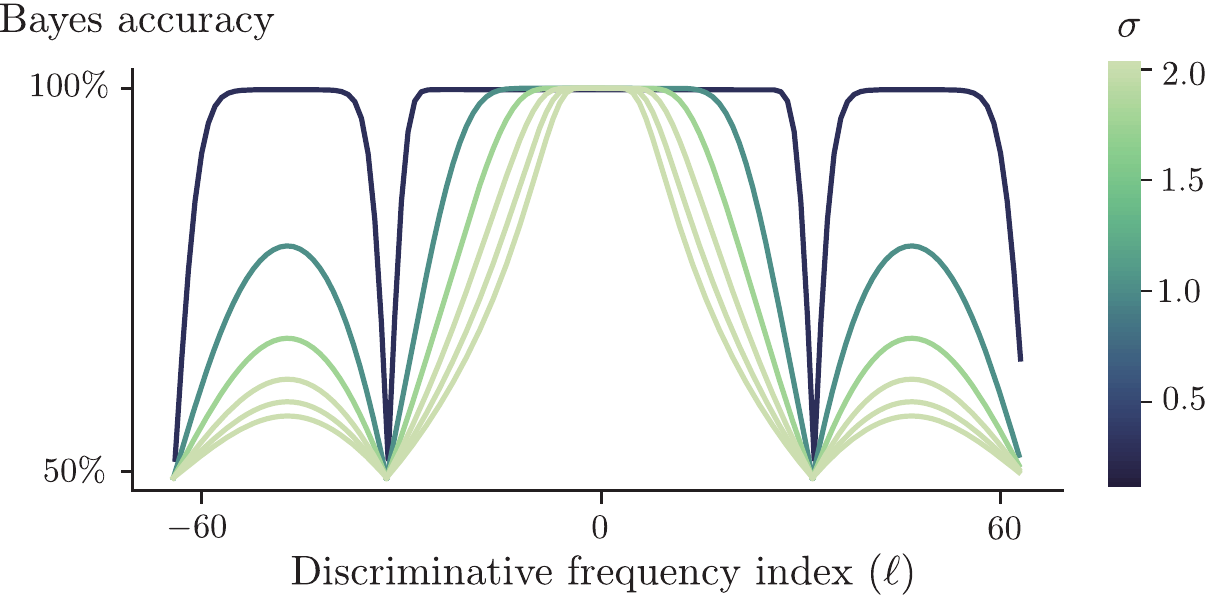}
    \caption{Optimal performance after an average pooling layer.}
    \label{fig:bayes_freq}
\end{subfigure}
\hfill
\begin{subfigure}{0.46\textwidth}
    \includegraphics[width=\textwidth]{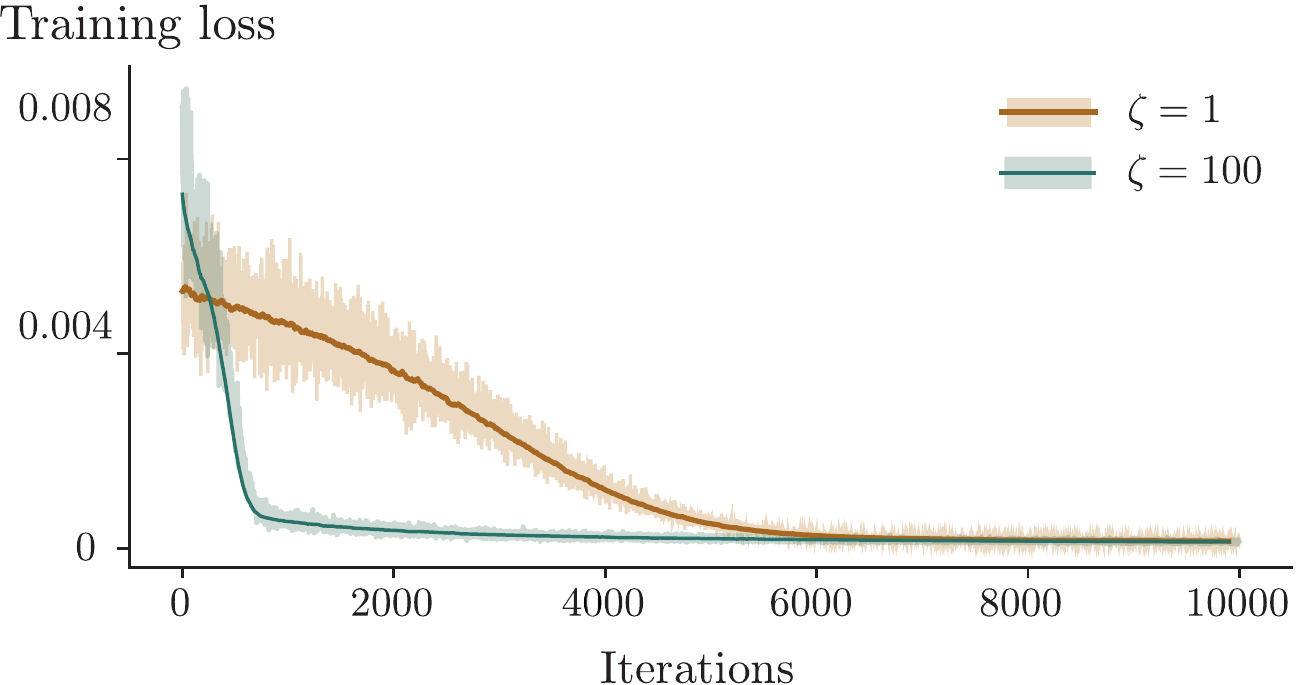}
    \caption{Training loss of the linear model of pooling.}
    \label{fig:optimization}
\end{subfigure}
\caption{Effects of pooling ($S=4, D=128$) on the inductive bias of a deep neural network.}
\end{figure}

\subsection{Anisotropic conditioning of the optimization landscape}\label{sec:condition}

Even if there is no information loss, the dependency of the optimization landscape on the discriminative direction can cause a network to show some bias towards the solutions that are better conditioned. This phenomenon can happen even on simple architectures. Hence, we illustrate it by studying an idealized deep linear model of the behaviour of pooling in the spectral domain.

In particular, we study the network $f_{\bm{\theta},\bm{\phi}}(\bm{x})=\bm{\theta}^T\bm{A}(\bm{m}\odot\bm{\phi}\odot\bm{x})$. In this model, $\bm{\phi}$ plays the role of the spectral response of a long stack of convolutional layers, $\bm{m}$ the spectral response of the pooling prefilter, and $\bm{\theta}$ the parameters of a fully connected layer at the output. 

For the sake of simplicity, we assume that the data follows $\bm{x}=\epsilon y\bm{e}_{\ell}+\bm{w}$ with isotropic noise $\bm{w}\sim \mathcal{N}(\bm{0}, \sigma^2\bm{I}_D)$. Note that for this family of datasets, the best achievable performance of this network is independent of the position of the discriminative feature $\ell$. Indeed, when the filter $\bm{\phi}$ takes the optimal value $\bm{\phi}[\ell]=1/\bm{m}[\ell]$ and $\bm{\phi}[t]=0$ for all $t\neq \ell$, the aliasing effect of $\bm{A}$ can be neglected.

We study the loss landscape when optimizing a quadratic loss, $J(\bm{\theta}, \bm{\phi}; \bm{x}, y)=(y-f_{\bm{\theta},\bm{\phi}}(\bm{x}))^2$. In this setting, the following lemma describes the statistics of the geometry of the loss landscape.

\begin{lemma}[Average curvature of the loss landscape]\label{lemma:curvature}
Assuming that the training parameters are distributed according to $\bm{\theta}\sim\mathcal{N}(\bm{0}, \sigma^2_{\bm{\theta}}\bm{I}_M)$ and  $\bm{\phi}\sim\mathcal{N}(\bm{0}, \sigma^2_{\bm{\phi}}\bm{I}_D)$, the average weight Hessian of the loss with respect to $\bm{\phi}$ satisfies $\mathbb{E}\,\nabla_{\bm{\phi}}^2J (\bm{\theta}, \bm{\phi};\bm{x}, y)=\underbrace{2\epsilon^2\bm{m}^2[\ell]\sigma^2_{\bm{\theta}}\operatorname{diag}(\bm{e}_{\ell})}_{\text{signal}}+\underbrace{2\sigma^2\sigma^2_{\bm{\theta}}\operatorname{diag}(\bm{m}^2)}_{\text{noise}}$.
\end{lemma}
\begin{proof}
See Sec.~B.2 of Supp. material\footnote{A similar result for $\bm{\theta}$ can be found in Sec.~B.2 of Supp. material.}.
\end{proof}

The curvature of the loss landscape can thus be decomposed into two terms: the curvature introduced by the discriminative signal, and the curvature introduced by the non-discriminative noise. A quantity that will control the speed of convergence of SGD will be the ratio between the curvature due to the signal component and the maximum curvature of the noise $\zeta(\ell)=\epsilon^2 \bm{m}^2[\ell]/\sigma^2 \max(\bm{m}^2)$~\cite{ghorbaniInvestigationNeuralNet}.

Intuitively, if $\zeta(\ell)\gg 1$, a small enough learning rate will quickly optimize the network in the direction of the optimal solution, avoiding big gradient oscillations caused by the non-discriminative components' curvature. On the contrary, if   $\zeta(\ell)\lesssim 1$, the speed of convergence will be much slower due to the big oscillations introduced by the greatly curved noise components (see Fig.~\ref{fig:optimization} and Fig.~\ref{fig:real_optimization}).

\section{NAD computation}
\label{sec:nad_computation}

The choice of the Fourier basis so far was almost arbitrary, and there is no reason to suspect that it will capture the full directional inductive bias of all CNNs. NADs characterize this general bias, but it is clear that trying to identify them by measuring the performance of a neural network on many linearly separable datasets parameterized by a random $\bm{v}$ would be extremely inefficient\footnote{For most of these datasets the network outputs the same performance as seen in Sec.~A.3 of Supp. material.}. It is therefore of paramount importance that we find another way to compute NADs without training.  

In this sense, we will study the behaviour of a given architecture when it tries to solve a very simple discriminative task: classifying two data samples, $(\bm{x}, +1)$ and $(\bm{x}+\bm{v}, -1)$. We call them a \emph{discriminative dipole}. Remarkably, studying this simple problem is enough to identify the NADs.

In general, given a discriminative dipole and a network $f_{\bm{\theta}}:\mathbb{R}^D\rightarrow \mathbb{R}$, parameterized by a general set of weights $\bm{\theta}$, we say that $f_{\bm{\theta}}$ has a high confidence in discriminating the dipole if it scores high on the metric $q_{\bm{\theta}}(\bm{v})=g\left(|f_{\bm{\theta}}(\bm{x})- f_{\bm{\theta}}(\bm{x}+\bm{v})|\right)$, where $g(t)$ can be any increasing function on $t\geq0$, e.g., $g(t)=t^2$. In practice, we approximate $q_{\bm{\theta}}(\bm{v})\approx g\left(|\bm{v}^T\nabla_{\bm{x}}f_{\bm{\theta}}(\bm{x})|\right)$ using a first-order Taylor expansion of $f_{\bm{\theta}}({\bm{x}}+\bm{v})$ around $\bm{x}$.

As shown in Fig.~\ref{fig:real_optimization}, the directional bias  can be identified based on the speed of convergence of a training algorithm. In the case of the dipole metric, this speed will depend on the size of $\|\nabla_{\bm{\theta}}q_{\bm{\theta}}(\bm{v})\|$. In expectation, this magnitude can be bounded by the following lemma.
\begin{lemma}\label{lemma:bound}
Let $g$ be any increasing function on $t>0$ with $|g'(t)|\leq\alpha |t|+\beta$, where $\alpha,\beta\geq0$. Then
\begin{equation}
   \mathbb{E}_{\bm{\theta}} \|\nabla_{\bm{\theta}}q_{\bm{\theta}}(\bm{v})\|\leq \alpha^2\sqrt{\mathbb{E}_{\bm{\theta}} |\bm{v}^T\nabla_{\bm{x}}f_{\bm{\theta}}(\bm{x})|^2}\sqrt{\mathbb{E}_{\bm{\theta}}\|\nabla^2_{\bm{\theta},\bm{x}}f_{\bm{\theta}}(\bm{x})\bm{v}\|^2}+\beta\mathbb{E}_{\bm{\theta}}\|\nabla^2_{\bm{\theta},\bm{x}}f_{\bm{\theta}}(\bm{x})\bm{v}\|.\label{eq:gradient_bound}
  \end{equation}
\end{lemma}
\begin{proof}

See Sec.~B.3 of Supp. material.
\end{proof}
 
The right-hand side of~\eqref{eq:gradient_bound} upper bounds $\mathbb{E}_{\bm{\theta}} \|\nabla_{\bm{\theta}}q_{\bm{\theta}}(\bm{v})\|$, and its magnitude with respect to $\bm{v}$ is controlled by the eigenvectors of $\mathbb{E}_{\bm{\theta}}\nabla_{\bm{x}}f_{\bm{\theta}}(\bm{x})\nabla^T_{\bm{x}}f_{\bm{\theta}}(\bm{x})$ and the expected right singular vectors of $\nabla^2_{\bm{\theta},\bm{x}}f_{\bm{\theta}}(\bm{x})$. We expect therefore that the NADs of an architecture are tightly linked to these vectors. The following example analyzes this relation for the deep linear network of Sec.~\ref{sec:condition}.

\begin{example*}
Let $\bm{\phi}\sim\mathcal{N}(\bm{0}, \sigma^2_{\bm{\phi}}\bm{I}_D)$ and $\bm{\theta}\sim\mathcal{N}(\bm{0}, \sigma^2_{\bm{\theta}}\bm{I}_M)$. The covariance of the input gradient of the linear model of pooling $f_{\bm{\theta},\bm{\phi}}$ is
\begin{equation}
    \mathbb{E}_{\bm{\theta},\bm{\phi}}\nabla_{\bm{x}}f_{\bm{\theta},\bm{\phi}}(\bm{x})\nabla^T_{\bm{x}}f_{\bm{\theta},\bm{\phi}}(\bm{x})=\sigma_{\bm{\phi}}^2\sigma_{\bm{\theta}}^2\operatorname{diag}\left(\bm{m}^2\right),
\end{equation}
and its eigenvectors are the canonical basis elements of $\mathbb{R}^D$, sorted by the entries of the prefilter $\bm{m}^2$. Surprisingly, the expected right singular vectors of $\nabla^2_{\bm{\theta},\bm{x}}f_{\bm{\theta}}(\bm{x})$ coincide with these eigenvectors,
\begin{equation}
    \mathbb{E}_{\bm{\theta},\bm{\phi}}\nabla^2_{(\bm{\theta},\bm{\phi}),\bm{x}}f_{\bm{\theta},\bm{\phi}}(\bm{x})^T\nabla^2_{(\bm{\theta},\bm{\phi}),\bm{x}}f_{\bm{\theta},\bm{\phi}}(\bm{x})=\left(\sigma^2_{\bm{\theta}}+\cfrac{\sigma^2_{\bm{\phi}}}{S}\right)\operatorname{diag}(\bm{m}^2).
\end{equation}
Again, this result agrees with what has been shown in Sec.~\ref{sec:condition} where the NADs are also ordered according to the entries of $\bm{m}^2$.
\end{example*}
\begin{proof}
See Sec.~C.1 of Supp. material.
\end{proof}

At this stage, it is important to highlight that these eigenvectors and singular vectors need not coincide, in general. However, as we will see in practice, these bases are surprisingly aligned for most networks, suggesting that the structure of these vectors is commonly rooted on some fundamental property of the architecture. Besides, although it could be argued that the bound in \eqref{eq:gradient_bound} is just an artefact of the choice of dipole metric, we provide below an alternative interpretation on the connection between $\mathbb{E}_{\bm{\theta}}\nabla_{\bm{x}}f_{\bm{\theta}}(\bm{x})\nabla^T_{\bm{x}}f_{\bm{\theta}}(\bm{x})$, $\nabla^2_{\bm{\theta},\bm{x}}f_{\bm{\theta}}(\bm{x})$ and NADs.

On the one hand, $\nabla^2_{\bm{\theta},\bm{x}}f_{\bm{\theta}}(\bm{x})$ can be interpreted as a magnitude that controls the network tendency to create a decision boundary along a given direction, i.e., its right singular values quantify the inclination of a network to align $\nabla_{\bm{x}}f_{\bm{\theta}}(\bm{x})$ with a discriminative direction $\bm{v}$. On the other hand, the eigenvalues of $\mathbb{E}_{\bm{\theta}}\nabla_{\bm{x}}f_{\bm{\theta}}(\bm{x})\nabla^T_{\bm{x}}f_{\bm{\theta}}(\bm{x})$ bound the \emph{a priori} hardness to find a solution discriminating in a given direction. Specifically, considering the quadratic case, i.e., $g(t)=t^2$, we can estimate the volume of solutions that achieve a certain dipole metric $\eta$, $\mathbb{P}\left(q_{\bm{\theta}}(\bm{v})\geq \eta\right)$. Indeed, an approximate bound to this volume using Markov's inequality depends only on the gradient covariance
\begin{equation}
    \mathbb{P}\left(q_{\bm{\theta}}(\bm{v})\geq \eta\right)\approx \mathbb{P}\left(\left(\bm{v}^T\nabla_{\bm{x}}f_{\bm{\theta}}(\bm{x})\right)^2\geq \eta\right)\leq \cfrac{\bm{v}^T\left(\mathbb{E}_{\bm{\theta}}\nabla_{\bm{x}}f_{\bm{\theta}}(\bm{x})\nabla^T_{\bm{x}}f_{\bm{\theta}}(\bm{x})\right)\bm{v}}{\eta}.\label{eq:volume}
\end{equation}
If the quadratic form is very low, the space of solutions achieving a certain $q_{\bm{\theta}}(\bm{v})$ is small. That is, it is much harder for a network to find solutions that optimize $q_{\bm{\theta}}(\bm{v})$ if the discriminative direction of the dipole is aligned with the eigenvectors associated to small eigenvalues of the gradient covariance.

Finally, note that the proposed eigendecomposition of the gradient covariance bares similarities with the techniques used to study neural networks in the mean-field regime~\cite{ganguli2016,deep_information,penningtonEmergenceSpectralUniversality2018}. These works study the effect of initialization and non-linearities on the Jacobian of inifinitely-wide networks to understand their trainability. In contrast, we analyze the properties of finite-size architectures and investigate the directionality of the singular vectors to explain the role of NADs in generalization. Analyzing the connections of NADs with these works will be subject of future research.

\subsection{NADs of CNNs}
\label{sec:nads_cnns}

\begin{figure*}[t]
\includegraphics[width=\textwidth]{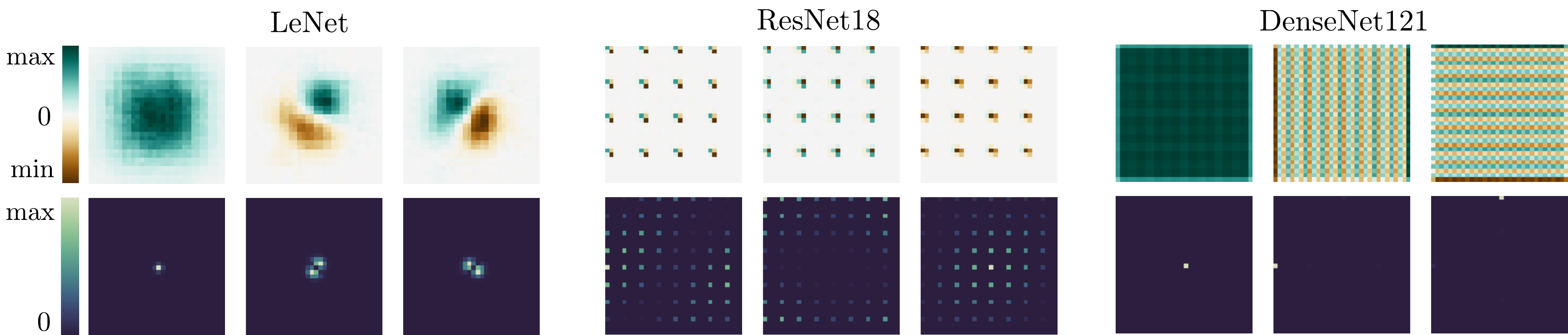}
\caption{First three NADs of state-of-the-art CNNs in computer vision (more in Sec.~D.1 of Supp. material). \textbf{Top row} shows NADs in pixel space and \textbf{bottom row} their energy in the Fourier domain.}
\label{fig:nads}
\end{figure*}

For most deep networks, however, it is not tractable to analytically compute these decompositions in closed form. For this reason, we can apply Monte-Carlo sampling to estimate them. As we mentioned above, the bases derived from $\nabla^2_{\bm{\theta},\bm{x}}f_{\bm{\theta}}(\bm{x})$ and $\nabla_{\bm{x}}f_{\bm{\theta}}(\bm{x})$ are surprisingly very similar for most networks. Nevertheless, we observe that the approximation of NADs through the eigendecomposition of the gradient covariance is numerically more stable\footnote{A detail description of the algorithmic implementation used to approximate NADs and examples of more NADs computed using both decompositions are given in Sec.~D.1 and Sec.~D.2 of Supp. material.}. For this reason, we use these in the remainder of the paper. 

Fig.~\ref{fig:nads} shows a few examples of NADs from several CNNs and illustrates their diversity. Focusing on these patterns, we can say that NADs act as a unique signature of an architecture. Remarkably, even though the supports of the energy in the Fourier domain of the first few NADs are included in the high accuracy regions of Fig.~\ref{fig:dfts}, not all NADs are sparse in the spectral domain. In particular, the NADs of a ResNet-18 look like combs of spikes in the Fourier domain. Similarly, the NADs do not follow a uniform ordering from low to high frequencies (cf.~DenseNet-121 in Fig.~\ref{fig:nads}). This suggests that each CNN relies on a unique set of features of training data to discriminate between different classes.

Analyzing how the exact construction of each architecture modulates the NADs of a network is out of the scope of this work. We believe that future research in this direction should focus on describing the individual contribution of different layers to understand the shape of NADs on CNNs.

\section{NADs and generalization}
In this section, we investigate the role of NADs on generalization in different scenarios. First, in the linearly separable setting, and later for CIFAR-10 dataset. In this sense, we will demonstrate that NADs are important quantities affecting the generalization properties of an architecture.

\subsection{Learning linearly separable datasets}

When a dataset is linearly separable by a single feature, the NADs completely characterize the performance of a network on this task. To show this, we replicate the experiments of Sec.~\ref{sec:spectral_bias}, but this time using the NADs to parameterize the different distributions. In Fig.~\ref{fig:nad_accs} we can see that the performance of these architectures monotonically decreases for higher NADs, and recall that in the Fourier basis (see Fig.~\ref{fig:dfts}), the performance did not monotonically decrease with frequency.

Observe as well that for the networks with lower rank of their gradient covariance (see Fig.~\ref{fig:nad_accs}), i.e., with a faster decay on its eigenvalues, the drop in accuracy happens at earlier indices and it is much more pronounced. In this sense, the ResNet-18 and DenseNet-121 that perform best on vision datasets such as ImageNet, ironically are the ones with the stronger bias on linearly separable datasets.

\begin{figure*}[t]
\includegraphics[width=\textwidth]{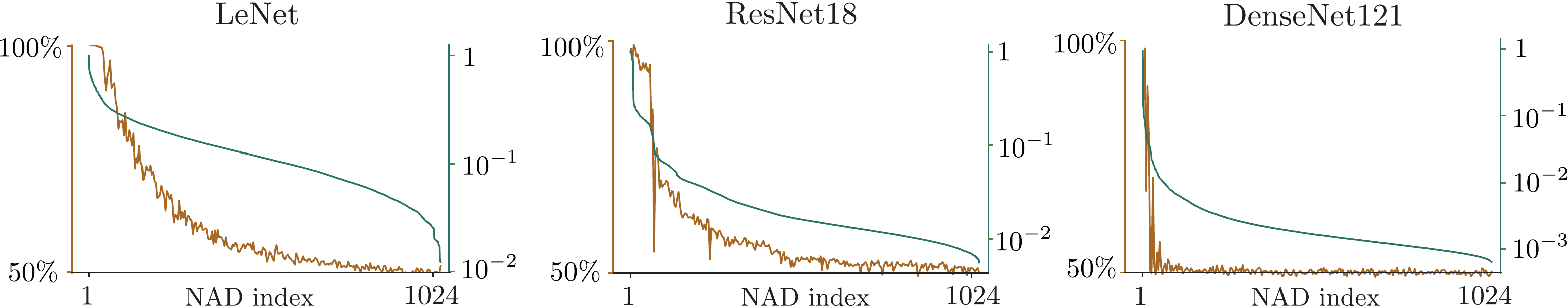}
\caption{\textbf{(Green)} Normalized covariance eigenvalues and \textbf{(brown)} test accuracies of common state-of-the-art CNNs trained on linearly separable distributions parameterized by their NADs.}
\label{fig:nad_accs}
\end{figure*}

\subsection{Learning CIFAR-10 dataset}

Finally, we provide two experiments that illustrate the role of NADs beyond linearly separable cases.

\paragraph{NADs define the order of selection of different discriminative features} First we borrow concepts from the data poisoning literature~\cite{frogs} as a way to probe the order in which features are selected by a given network. In particular, we do this by modifying all images in the CIFAR-10 training set to include a highly discriminative feature (carrier) aligned with a certain NAD. We repeat this experiment for multiple NADs and measure the test accuracy on the original CIFAR-10 test set.

\begin{figure}[t]
\begin{minipage}[t]{.54\textwidth}
  \centering
\includegraphics[width=\textwidth]{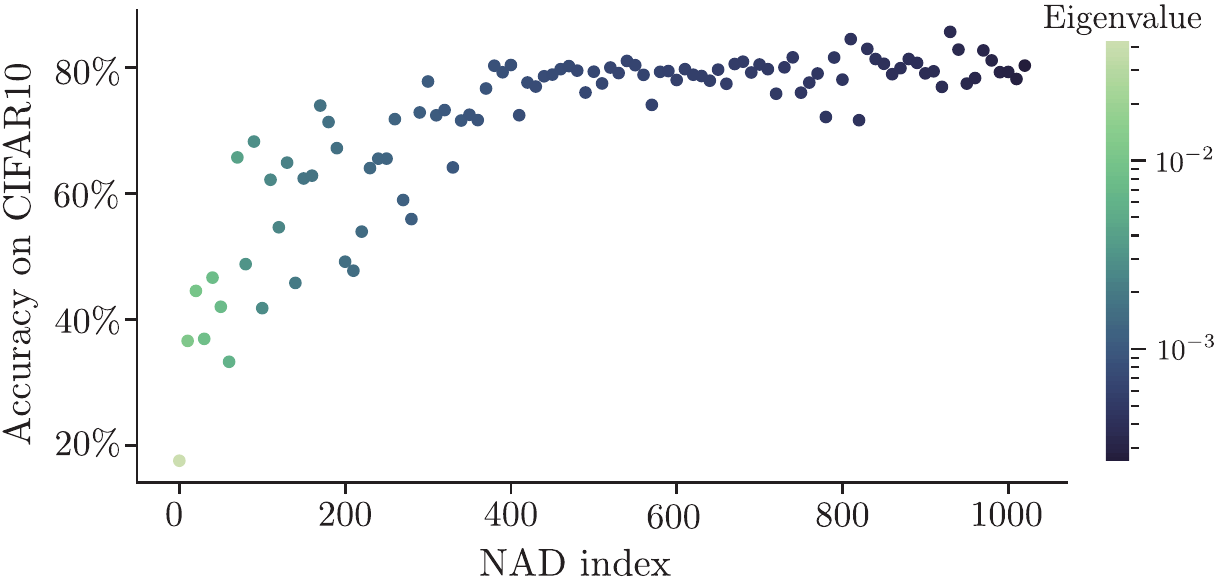}
\caption{Accuracy on CIFAR-10 of a ResNet-18 when trained on multiple versions of poisoned data with a carrier ($\epsilon=0.05$) at different NAD indices.}
\label{fig:poison}
\end{minipage}
\hfill
\begin{minipage}[t]{.39\textwidth}
\centering
\includegraphics[width=\textwidth]{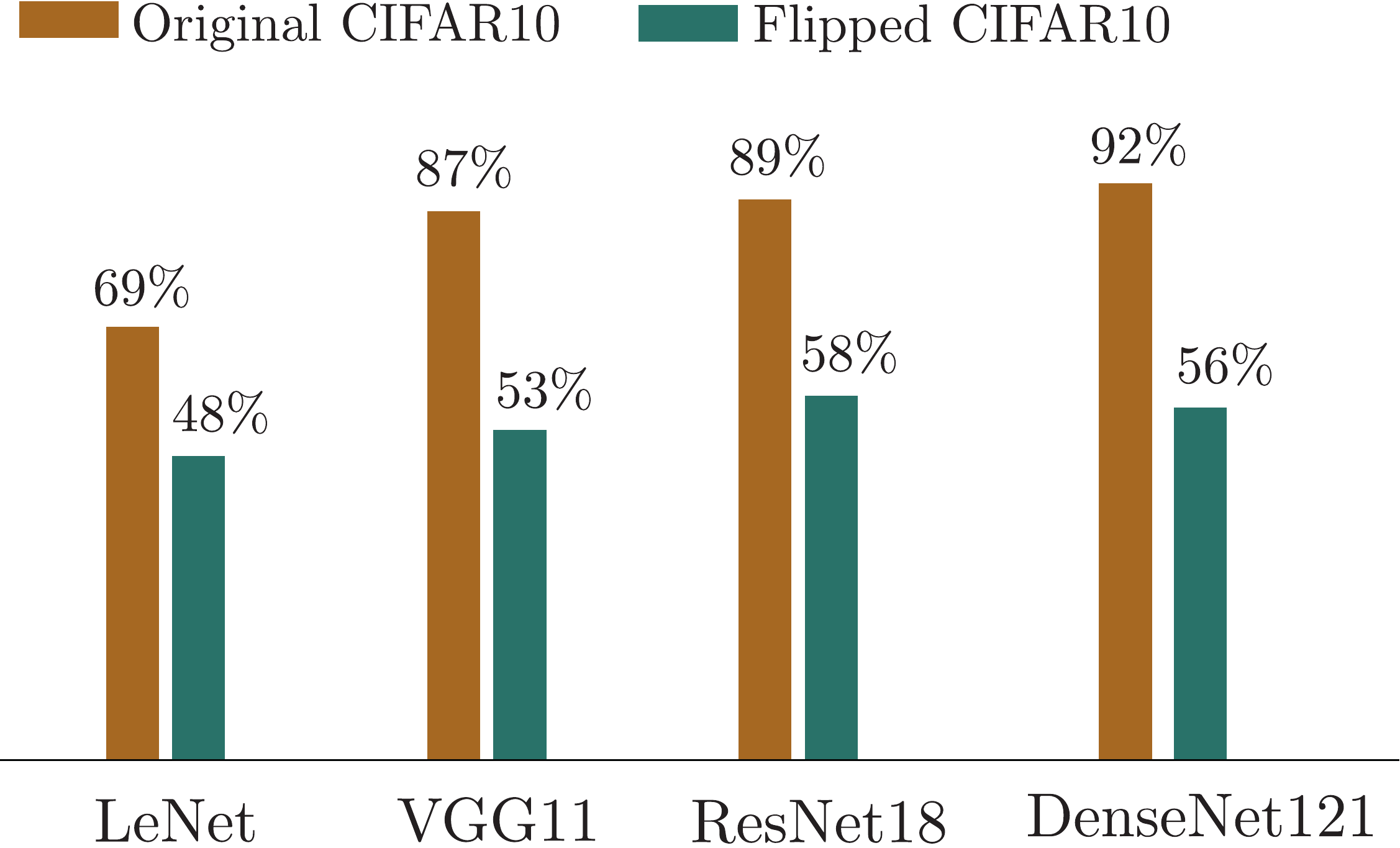}
\caption{Test accuracies of state-of-the-art CNNs on the standard CIFAR-10 dataset and its flipped version.}
\label{fig:nad_rotation}
\end{minipage}
\end{figure}

An easy way to introduce a poisonous carrier on a sample is to substitute its content on a given direction by $\pm\epsilon$. CIFAR-10 has $10$ classes and three color channels. Therefore, we can use two consecutive NADs applied on the different channels to encode a carrier that can poison this dataset. Note that, for any $\epsilon>0$, this small modification on the training set renders the training set linearly separable using only the poisonous features. But, a classifier that only uses these features will not be able to generalize to the unpoisoned CIFAR-10 test set.

Fig.~\ref{fig:poison} shows the result of this experiment when the carriers are placed at the $i$th and $(i+1)$th NADs. For carriers placed at the first NADs the test accuracy is very low, showing that the network ignores most generalizing features. On the other hand, when the carrier is placed at the end of the sequence, where the extra feature is harder to learn (cf.~Fig.\ref{fig:nad_accs}), the network can almost perfectly generalize.

A possible explanation of this behavior could be that during training, among all possible separating solutions, a network converges to the one that can discriminate the training data using the lowest NAD indices, i.e.,~using features spanned by the first NADs. In this sense, when a carrier is placed at a given NAD index, the network can only identify those generalizing features spanned by the NADs before the carrier, and ignores all those after it.

\paragraph{NADs are necessary for generalization} To further support the previous explanation, we investigate the role of NADs as filters of discriminating solutions. In particular, we test the possible positive synergies arising from the alignment of NADs with the generalizing features of the training set. Specifically, we train multiple CNNs using the same hyperparameters on two representations of CIFAR-10: the original representation, and a new one in which we flip the representation of the data in the NAD basis. That is, for every sample $\bm{x}$ in the training and test sets we compute $\bm{x}'=\bm{U}\operatorname{flip}(\bm{U}^T\bm{x})$, where $\bm{U}$ represents a matrix with NAD vectors as its columns. Note that applying this transformation equates to a linear rotation of the input space and has no impact on the information of the data distribution. In fact, training on both representations yields approximately $0\%$ training error.

Fig.~\ref{fig:nad_rotation} shows the result of these experiments where we see that the performance of the networks trained on the flipped datasets is significantly lower than those on the original CIFAR-10. As demonstrated by the low accuracies on the flipped datasets, misaligning the inductive bias of these architectures with the datasets makes them prone to overfit to non-generalizing and spurious ``noise''. We see this effect as a supporting evidence that through the years the community has managed to impose the right inductive biases in deep neural architectures to classify the standard vision benchmarks.

\section{Conclusion}
In this paper we described a new type of model-driven inductive bias that controls generalization in deep neural networks: the directional inductive bias. We showed that this bias is encoded by an orthonormal set of vectors for each architecture, which we coined the NADs, and that these characterize the selection of discriminative features used by CNNs to separate a training set. In \cite{zhangUnderstandingDeepLearning2017}, researchers highlighted that a neural network memorizes a dataset when this has no discriminative information. In our work we complement this observation, and show that a network may prefer memorization over generalization, even when there exists a highly discriminative feature in the dataset. Surprisingly, this phenomenon is not only attributable to some property of the data, but also to the structure of the architecture.

Future research should focus on providing a better theoretical understanding of the mechanisms that determine the NADs, but also on describing their role on the dynamics of training. Extending the NAD discovery algorithms to other families of architectures like graph neural networks~\cite{defferrardConvolutionalNeuralNetworks} or transformers~\cite{vaswaniAttentionAllYou2017} would be a natural next step. All in all, we believe that our findings can have potential impacts on future research in designing better architectures and AutoML~\cite{automl}, paving the way for better aligning the inductive biases of deep networks with \emph{a priori} structures on real data.

Finally, it is important to note that our results mostly apply to cases in which the data was fully separable, i.e. there was no label noise. And even more specifically, to the linearly separable case. In this sense, it still remains an open problem to understand how the directional inductive bias of deep learning influences neural networks trying to learn non-separable datasets.

\section*{Broader Impact}
In this work we reveal the directional inductive bias of deep learning and describe its role in controlling the type of functions that neural networks can learn. The algorithm that we introduced to characterize it can help understand the reasons for the success or alternatively the modes of failure of most modern CNNs. Our work is mainly fundamental in the sense that it is not geared towards an application, but theory always has some downstream implications on the society as enabler of future applications.

We see potential applications of our work on AutoML~\cite{automl} as the main positive impact of our research. In particular, we believe that incorporating prior knowledge into the neural architecture search loop~\cite{zophNEURALARCHITECTURESEARCH2017} can cut most computational and environmental costs of this procedure. Specifically, with the current trend in deep learning towards building bigger and computationally greedier models~\cite{brownLanguageModelsAre2020}, the impact of machine learning on the environment is becoming a pressing issue~\cite{rolnickTacklingClimateChange2019}. Meanwhile, this trend is raising the bar on the needed resources to use these models and research in deep learning is getting concentrated around a few big actors. In this sense, we believe that gaining a better understanding of our current models will be key in circumventing the heavy heuristics necessary to deploy deep learning today, thus enabling the democratization of this technology~\cite{AIGlobalGovernance}.

On the other hand, we see the main possible negative implication of our work in the malicious use of NADs to boost the adversarial capacity of new evasion/backdoor attacks~\cite{papernotScienceSecurityPrivacy2016}. This could potentially exploit the sensitivity of neural networks to NADs to generate more sophisticated adversarial techniques. Machine learning engineers should be aware of such vulnerabilities when designing new architectures, especially for safety-critical applications.

\section*{Acknowledgments}
We thank Maksym Andriushchenko, Hadi Daneshmand, and Cl{\'e}ment Vignac, for their fruitful discussions and feedback. This work has been partially supported by the CHIST-ERA program under Swiss NSF Grant 20CH21\_180444, and partially by Google via a Postdoctoral Fellowship and a GCP Research Credit Award.

\bibliographystyle{ieeetr}
\bibliography{main.bib}

\newpage

\renewcommand\thefigure{S\arabic{figure}} 
\renewcommand\thetable{S\arabic{table}} 

\appendix

\section{Experiments on linearly separable datasets}
\subsection{General training setup}
Regarding the construction of the synthetic datasets used for the experiments of Sec.~\ref{sec:spectral_bias} and Sec.~\ref{sec:nads_cnns}, recall that $\mathcal{D}(\bm{v})$ is a linearly separable distribution parameterized by a unit vector $\bm{v}\in\mathbb{S}^{D-1}$, such that any sample $(\bm{x}, y)\sim\mathcal{D}(\bm{v})$ satisfies $\bm{x}=\epsilon y\bm{v}+\bm{w}$, with noise $\bm{w}\sim\mathcal{N}\left(\bm{0}, \sigma^2(\bm{I}_D-\bm{v}\bm{v}^T)\right)$ orthogonal to the direction $\bm{v}$, and with label $y$ sampled from $\{-1, +1\}$ with equal probability. An illustration of such dataset is shown in Fig.~\ref{fig:linear_dataset_illustration}.

\begin{figure}[h!]
\begin{center}
\includegraphics[width=0.3\textwidth]{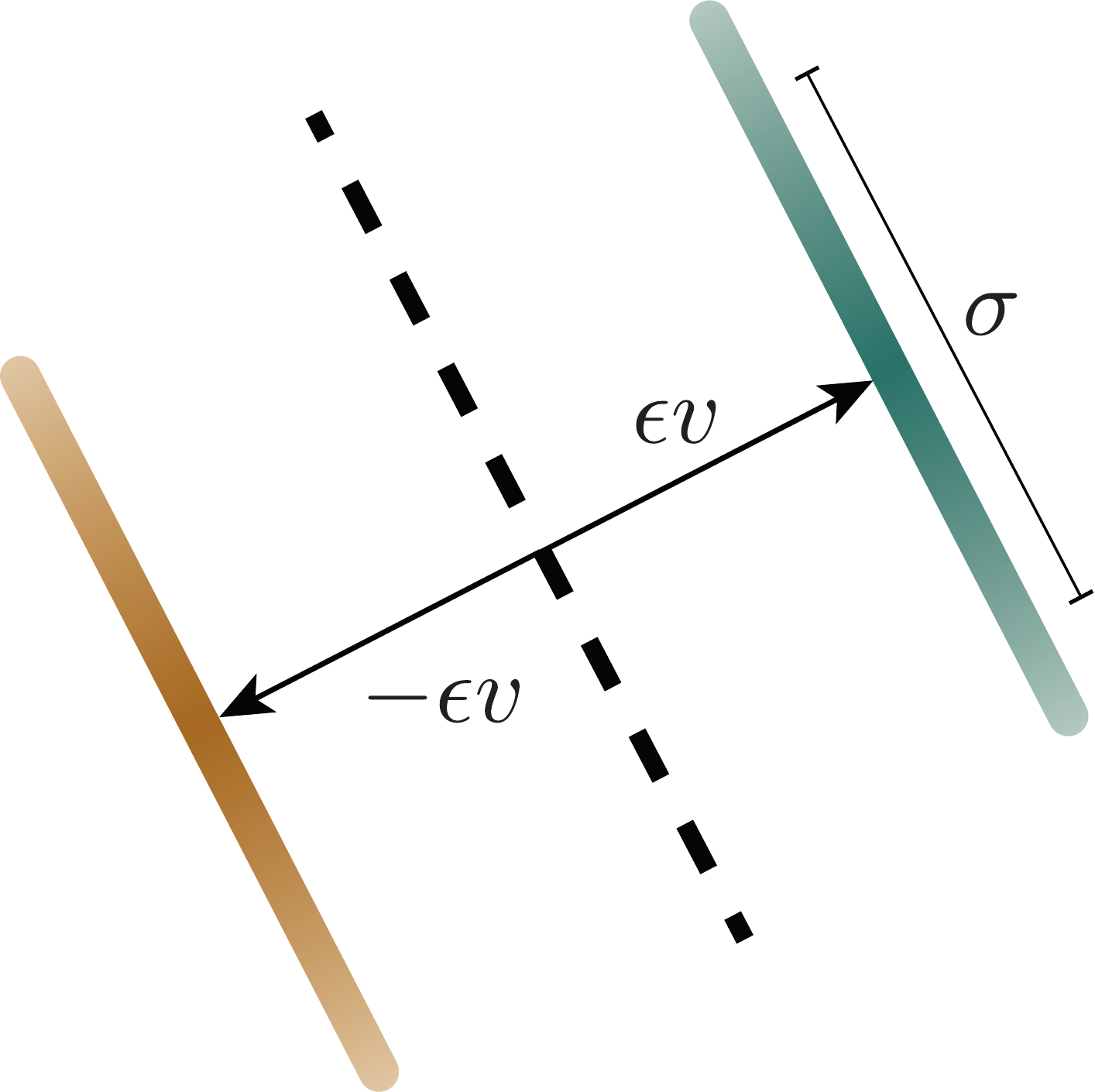}
\caption{Schematic of the parameters of $\mathcal{D}(\bm{v})$.}
\label{fig:linear_dataset_illustration}
\end{center}
\end{figure}

In general, the generated synthetic data correspond to $32\times32$ grayscale images, with the standard settings being $10,000$ training samples, $10,000$ test samples, and $\epsilon=1$. The value of $\sigma$ varies depending on the experiment under study.

Regarding the setup and parameters for training the networks used for the experiments of Sec.~\ref{sec:spectral_bias} and Sec.~\ref{sec:nads_cnns}: they were all trained for $20$ epochs, on batches of size $128$, minimizing a Cross-Entropy loss using SGD with a linearly decaying learning rate (max lr. $0.5$) and without any explicit regularization.

At this point let us note that we did not perform any extensive hyperparameter tuning to arrive at this configuration. In fact, we empirically observed that such parameters were good enough to reveal the quantities of interest (i.e.,~directional inductive bias) and did not tune them further. In general, all of our observations are relative in the sense that we do not focus on the exact values, e.g., test accuracy or training iterations (except reaching almost zero training loss), but their relative differences for different distributions.

\subsection{Experiments on DFT basis}
\subsubsection{Basis generation}
Recall that the DFT $\mathcal{F}:\mathbb{C}^D\rightarrow\mathbb{C}^D$ is a complex linear operator acting in the complex plane. For this reason, the basis obtained from transforming the canonical basis through the DFT, i.e., $\bm{v}_i=\mathcal{F}(\bm{e}_i)$ is a complex basis. In this work we are interested in dealing with real signals, and as such we need to modify this basis such that it is an orthonormal basis of the real space $\mathbb{R}^D$. 

We can do that by leveraging the conjugate symmetry of the DFT of real signals. Let $\bm{x}\in\R^D$ with Fourier transform $\hat{\bm{x}}=\mathcal{F}(\bm{x})\in\mathbb{C}^D$. Then,
\begin{equation*}
    \hat{\bm{x}}\llbracket t\rrbracket_{D} = \hat{\bm{x}}^\star\llbracket -t\rrbracket_{D},
\end{equation*}
where $\hat{\bm{x}}^\star$ represents the complex conjugate of $\hat{\bm{x}}$. This means that, for real signals, half of the DFT is redundant, and one can use only $\lfloor D/2\rfloor + 1$ complex coefficients to represent a real signal. We are interested in obtaining a basis of $\mathbb{R}^D$ which is sparse in the Fourier domain. However, note that for any index $t$, $\langle \mathcal{F}^{-1}\left(\bm{e}_t\right),\mathcal{F}^{-1}\left(j\bm{e}_t\right)\rangle = 0$, with $j=\sqrt{-1}$. For this reason, we can create a basis of $\mathbb{R}^D$ using $\lfloor D/2\rfloor + 1$ real coefficients and $\lfloor D/2\rfloor$ imaginary coefficients by exploiting their conjugate symmetries, i.e.,
\begin{align*}
    \bm{v}^\text{Re}_i&=\cfrac{1}{\sqrt{2}}\,\mathcal{F}^{-1}\left(\bm{e}_{\llbracket i\rrbracket_D} + \bm{e}_{\llbracket -i\rrbracket_D}\right)& i=0,\dots, \lfloor D/2\rfloor\\
    \bm{v}^\text{Im}_i&=\cfrac{1}{\sqrt{2}}\,\mathcal{F}^{-1}\left(j\bm{e}_{\llbracket i\rrbracket_D} - j\bm{e}_{\llbracket -i\rrbracket_D}\right)& i=1,\dots ,\lfloor D/2\rfloor
\end{align*}

Fortunately, most numerical linear algebra libraries avoid the need to keep track of these symmetries and include some routine to directly compute the Fourier transform and its inverse on real signals (RFFT). This is especially useful on bidimensional signals, like images, where the RFFT of a signal has $D\times \lfloor D/2\rfloor + 1$ complex coefficients. Nevertheless, despite the redundancies, it is a common convention in the image processing community to plot the full Fourier spectrum of an image including positive and negative frequencies (indices). In our plots, we follow this convention, and artificially create the symmetries on the negative indices to ease readability\footnote{For more information about the properties of the 2D-DFT, we refer the reader to~\cite{gonzalezDigitalImageProcessing2017}.}.

All the results that we have shown so far using the DFT basis show only the results for the directions obtained from manipulating the real coefficients. Nevertheless, the results do not change in nature when one repeats them on the imaginary elements as well. We provide Fig.~\ref{fig:dfts_imaginary} as a validation of this, where we repeated the same experiment as in Fig.~\ref{fig:dfts} but using the directions parameterized by the imaginary coefficients, i.e., $\bm{v}_i^{\text{Im}}$. Note that, because the number of basis vectors parameterized by the imaginary coefficients is smaller, there are four gaps in Fig.~\ref{fig:dfts_imaginary}. These are just artefacts of the visualization, as these distributions do not exist in reality.

\begin{figure*}[h!]
\begin{center}
\includegraphics[width=\textwidth]{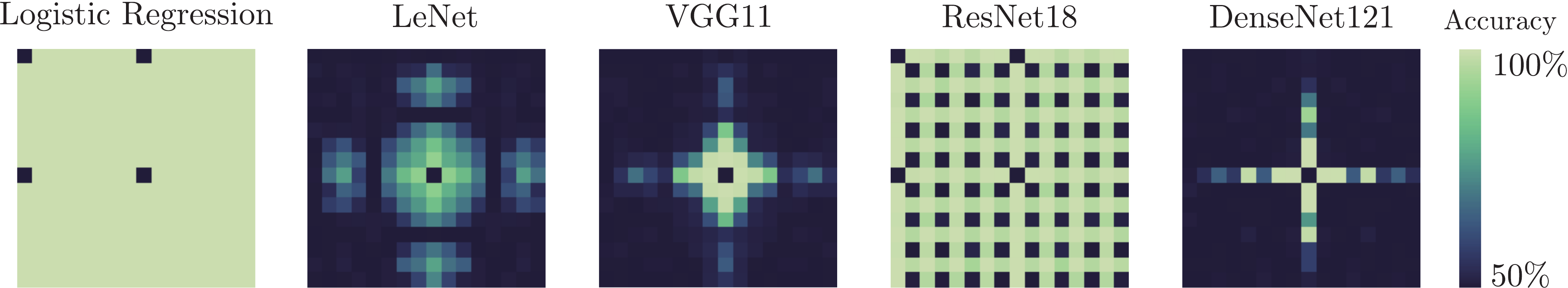}
\caption{Imaginary part of DFT}
\label{fig:dfts_imaginary}
\end{center}
\end{figure*}

\subsubsection{Different noise levels}
Fig.~\ref{fig:dft_sigmas} illustrates the test accuracies of various architectures under different noise levels $\sigma$. Regardless the noise level, a logistic regression can always perfectly generalize to the test data. On the contrary, LeNet seems to fail to generalize to a few distributions even in the absence of noise, while the noisier the data the more its performance degrades. The other architectures exhibit similar behaviour: they properly generalize when there is no noise, while their performance drops as the noise level increases. Finally, note that ResNet-18 seems to be slightly more robust to noise compared to the other CNNs (cf.~Fig.~\ref{fig:dft_sigmas_1} with $\sigma=1$).

\begin{figure*}[h!] 
\begin{subfigure}{\textwidth}
    \includegraphics[width=\textwidth]{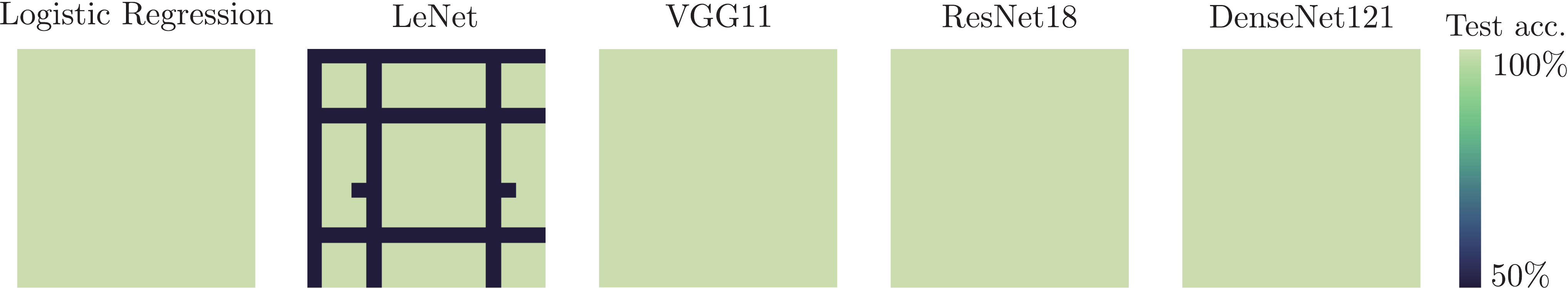}
    \caption{Test accuracies for $\sigma=0$.}
\end{subfigure}

\begin{subfigure}{\textwidth}
    \includegraphics[width=\textwidth]{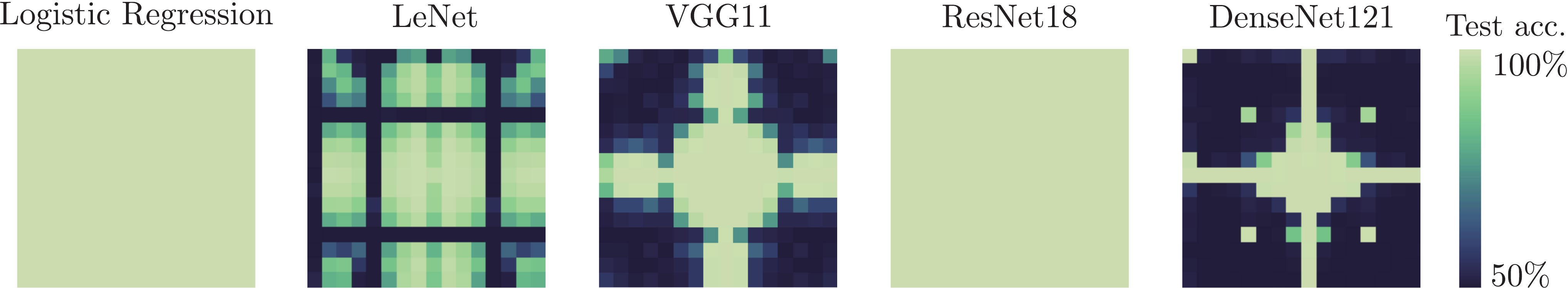}
    \caption{Test accuracies for $\sigma=1$.}
    \label{fig:dft_sigmas_1}
\end{subfigure}

\begin{subfigure}{\textwidth}
    \includegraphics[width=\textwidth]{test_dfts.eps}
    \caption{Test accuracies for $\sigma=3$.}
\end{subfigure}
\caption{Test accuracies using different training sets drawn from $\mathcal{D}(\bm{v})$ ($\epsilon=1$, with $10,000$ training samples and $10,000$ test samples) for different levels of $\sigma$. Directions $\bm{v}$ taken from the basis elements of the 2D-DFT. Each pixel corresponds to a linearly separable dataset.}
\label{fig:dft_sigmas}
\end{figure*}

\subsection{Experiments on random basis}
As mentioned in Sec.~\ref{sec:nad_computation}, trying to identify the NADs of an architecture by measuring its performance on many linearly separable datasets parameterized by a random direction $\bm{v}$, would be extremely inefficient. To demonstrate this, we repeat the same experiment performed in Sec.~\ref{sec:spectral_bias}, but instead of constructing the training sets $\mathcal{D}(\bm{v})$ using $\bm{v}$s taken from the 2D-DCT basis elements, each $\bm{v}$ now corresponds to a basis element of a random orthonormal matrix $\bm{U}\in \operatorname{SO}(D)$.

The results of this experiment are illustrated in Fig.~\ref{fig:random_basis}. Indeed, it is clear that such procedure will never be able to reveal the directional inductive bias of an architecture: for most of the datasets the networks output the same performance, thus it is impossible to interpret if these directions are aligned with the directional inductive bias of the architecture under study. 

\begin{figure}[h!]
    \begin{center}
    \begin{subfigure}{0.49\textwidth}
        \includegraphics[width=\textwidth]{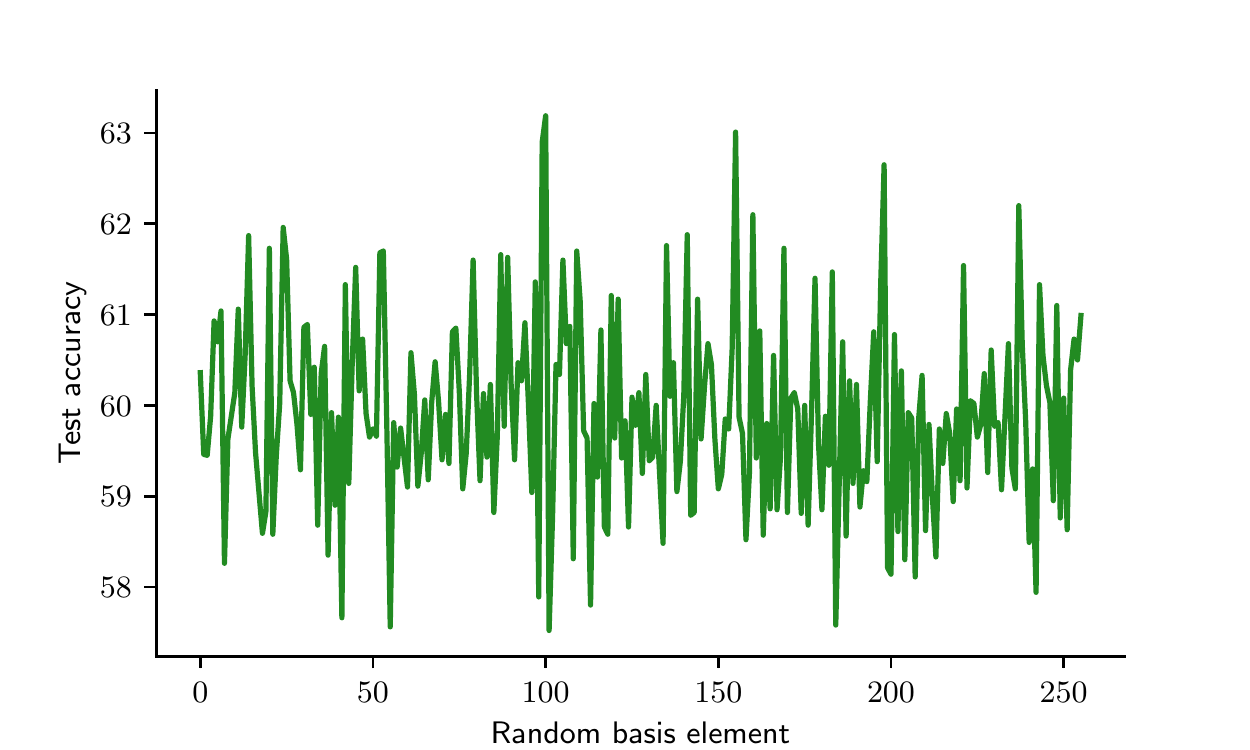}
        \caption{LeNet}
    \end{subfigure}
    \begin{subfigure}{0.49\textwidth}
        \includegraphics[width=\textwidth]{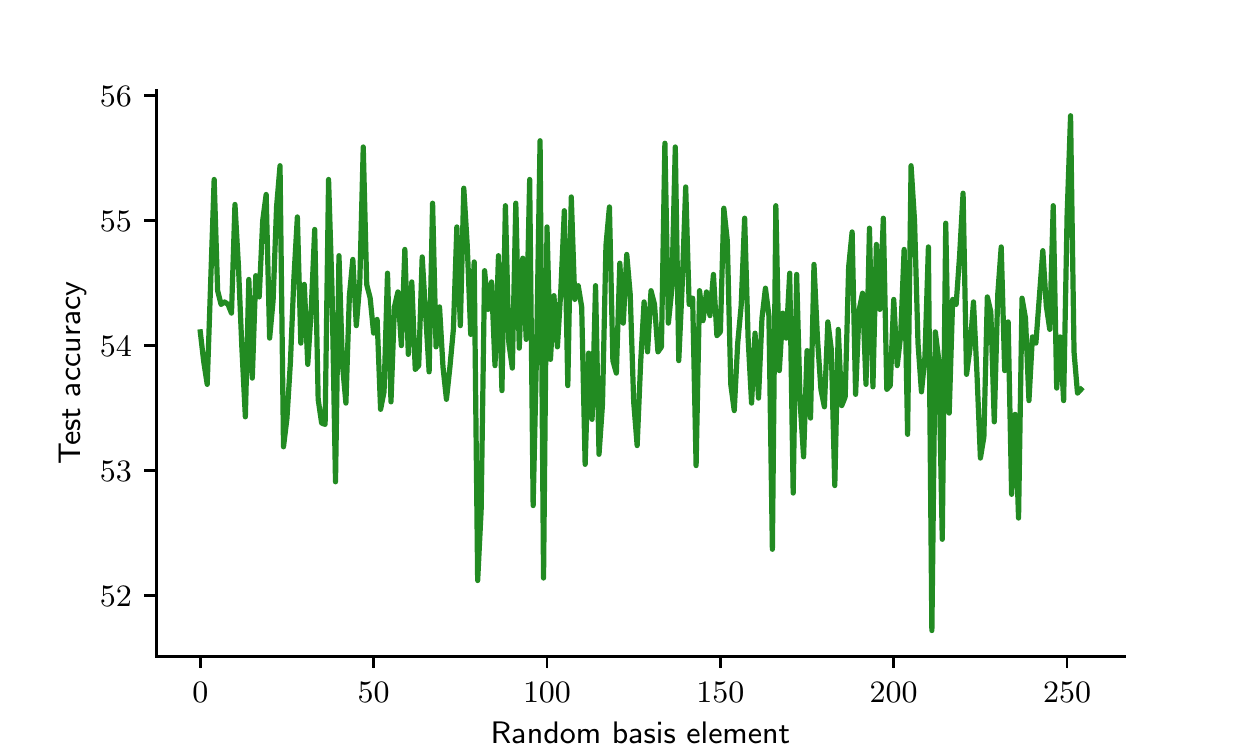}
        \caption{ResNet-18}
    \end{subfigure}
    \caption{Test accuracy of two CNNs trained using different training sets drawn from $\mathcal{D}(\bm{v})$ ($\epsilon=1$, and $\sigma=3$) with orthogonal random $\bm{v}$.}
    \label{fig:random_basis}
    \end{center}
\end{figure}

\clearpage
\newpage

\section{Deferred proofs}
\subsection{Proof of Theorem~\ref{thm:bayes}}

We give here the proof of Theorem~\ref{thm:bayes} stating the Bayes optimal classification accuracy achieved on a linearly separable distribution transformed through a linear pooling layer. We restate the theorem to ease readability.

\begin{theorem*}[Bayes optimal classification accuracy after pooling]
The best achievable accuracy on the distribution of $(\bm{z},y)$ can be written as
\begin{equation*}
    1-\mathcal{Q}\left(\cfrac{\epsilon}{2\sigma}\,\gamma(\ell)\right)\quad \text{with}\quad \gamma^2(\ell)=\cfrac{S|\hat{\bm{m}}[\ell]|^2}{\sum_{k=1}^{S-1}\left|\hat{\bm{m}}\llbracket\ell+k\cdot M\rrbracket_D\right|^2},
\end{equation*}
and $\mathcal{Q}(\cdot)$ representing the tail distribution function of the standard normal distribution. 
\end{theorem*}
\begin{proof}
Without loss of generality, let $(\bm{x},y)$ be a random sample with $y\sim\mathcal{U}\{-1,1\}$ and whose Fourier transform satisfies
\begin{equation*}
\hat{\bm{x}}=\epsilon y\bm{e}_\ell+\hat{\bm{w}}\qquad\text{with}\qquad\hat{\bm{w}}\sim\mathcal{CN}(\bm{0}, \operatorname{diag}(\bm{\sigma}^2)),
\end{equation*}
where $\mathcal{CN}(\bm{0}, \operatorname{diag}(\bm{\sigma})^2)$ denotes a circularly symmetric complex Gaussian distribution with complex covariance $\operatorname{diag}(\bm{\sigma}^2)$.

Because all entries of $\hat{\bm{x}}$ are uncorrelated, the best accuracy on the distribution of $(\bm{x},y)$, $\alpha_{\text{opt}}$,  would be the same as that of the distribution of $(\mathfrak{R}(\hat{\bm{x}}[\ell]),y)$, i.e., $\mathfrak{R}(\hat{\bm{x}}[\ell])|y=+1\sim \mathcal{N}(\epsilon,\bm{\sigma}^2[\ell]/2)$ and $\mathfrak{R}(\hat{\bm{x}}[\ell])|y=-1\sim \mathcal{N}(-\epsilon,\bm{\sigma}^2[\ell]/2)$. Hence,
\begin{equation*}
    \alpha_{\text{opt}}=1-\mathcal{Q}\left(\cfrac{\sqrt{2}\epsilon}{2\bm{\sigma}[\ell]}\right).
\end{equation*}

Nevertheless, we are interested on the accuracy on the distribution of $(\bm{z},y)$, when $(\bm{x},y)\sim\mathcal{D}(\bm{v}_{\ell})$ with $\bm{v}_\ell=\mathcal{F}(\bm{e}_{\ell})$, whose spectrum satisfies
\begin{equation*}
    \hat{\bm{z}}=\epsilon y\hat{\bm{m}}[\ell]\bm{e}'_{\llbracket\ell\rrbracket_M}+\operatorname{diag}(\hat{\bm{m}})\hat{\bm{w}}
\end{equation*}
with $\bm{e}'_{\llbracket\ell\rrbracket_M}\in\mathbb{C}^M$ the $(\ell\operatorname{mod} M)$th canonical basis vector of $\mathbb{R}^M$, and $\operatorname{diag}(\hat{\bm{m}})\hat{\bm{w}}\sim \mathcal{CN}(\bm{0},\operatorname{diag}(\bm{\xi}))$ with
\begin{equation*}
    \bm{\xi}^2\llbracket\ell\rrbracket_M=\cfrac{\sigma^2}{S}\,\sum_{k=1}^{S-1}\left|\hat{\bm{m}}\llbracket\ell+k\cdot M\rrbracket_D\right|^2.
\end{equation*}
Again, the only signal component is at $\hat{\bm{z}}\llbracket\ell\rrbracket_M$. Hence, if we write
\begin{equation*}
    \gamma^2(\ell)=\cfrac{S|\hat{\bm{m}}[\ell]|^2}{\sum_{k=1}^{S-1}\left|\hat{\bm{m}}\llbracket\ell+k\cdot M\rrbracket_D\right|^2},
\end{equation*}
and finally the accuracy of the Bayes optimal classifier on the distribution of $(\bm{z}, y)$ can be explicitly described by
\begin{equation*}
    \alpha(\ell)=1-\mathcal{Q}\left(\cfrac{\sqrt{2}\epsilon|\hat{\bm{m}}[\ell]|}{2\bm{\xi}\llbracket\ell\rrbracket_M}\right)=1-\mathcal{Q}\left(\cfrac{\sqrt{2}\epsilon}{2\sigma}\,\gamma(\ell)\right).
\end{equation*}

\end{proof}

\subsection{Proof of Lemma~\ref{lemma:curvature}}\label{sec:proof_lemma_1}
We detail here the proof of Lemma~\ref{lemma:curvature} describing the average curvature of the loss landscape for the deep linear network $f_{\bm{\theta}, \bm{\phi}}(\bm{x})=\bm{\theta}^T\bm{A}(\bm{m}\odot\bm{\phi}\odot\bm{x})$ when optimizing the quadratic loss $J(\bm{\theta},\bm{\phi};\bm{x},y)=(y-f_{\bm{\theta}, \bm{\phi}}(\bm{x}))^2$.

\begin{lemma*}[Average curvature of the loss landscape]
Assuming that the training parameters are distributed according to $\bm{\theta}\sim\mathcal{N}(\bm{0}, \sigma^2_{\bm{\theta}}\bm{I}_M)$ and  $\bm{\phi}\sim\mathcal{N}(\bm{0}, \sigma^2_{\bm{\phi}}\bm{I}_D)$, the average weight Hessian of the loss with respect to $\bm{\phi}$ satisfies $\mathbb{E}\,\nabla_{\bm{\phi}}^2J (\bm{\theta}, \bm{\phi};\bm{x}, y)=2\epsilon^2\bm{m}^2[\ell]\sigma^2_{\bm{\theta}}\operatorname{diag}(\bm{e}_{\ell})+2\sigma^2\sigma^2_{\bm{\theta}}\operatorname{diag}(\bm{m}^2)$.
\end{lemma*}
\begin{proof}
Let us start by computing the gradients for a generic loss $J (\bm{\theta}, \bm{\phi};\bm{x}, y)=q(z,y)$ with $z=f_{\bm{\theta},\bm{\phi}(\bm{x})}$
\begin{align*}
    \nabla_{\bm{\theta}} J(\bm{\theta}, \bm{\phi}; \bm{x}, y)&=q'(z,y)\bm{A}(\bm{\phi}\odot\bm{m}\odot\bm{x})\\
    \nabla_{\bm{\phi}} J(\bm{\theta}, \bm{\phi}; \bm{x}, y)&=q'(z,y)(\bm{A}^T\bm{\theta})\odot(\bm{m}\odot\bm{x}).
\end{align*}
Therefore, the second derivatives are
\begin{align*}
    \nabla^2_{\bm{\theta}}J(\bm{\theta}, \bm{\phi}; \bm{x}, y)&=q''(z,y)\bm{A}(\bm{\phi}\odot\bm{m}\odot\bm{x})(\bm{A}(\bm{\phi}\odot\bm{m}\odot\bm{x}))^T\\
    \nabla^2_{\bm{\phi}}J(\bm{\theta}, \bm{\phi}; \bm{x}, y)&=q''(z,y)(\bm{A}^T\bm{\theta})\odot(\bm{m}\odot\bm{x})((\bm{A}^T\bm{\theta})\odot(\bm{m}\odot\bm{x}))^T\\
    \nabla^2_{\bm{\theta},\bm{\phi}}J(\bm{\theta}, \bm{\phi}; \bm{x}, y)&=q''(z,y)((\bm{A}^T\bm{\theta})\odot(\bm{m}\odot\bm{x}))(\bm{A}(\bm{\phi}\odot\bm{m}\odot\bm{x}))^T\nonumber+\\
    &\quad+q'(z,y)\bm{A}\operatorname{diag}(\bm{x}\odot\bm{m})\\
    \nabla^2_{\bm{\phi},\bm{\theta}}J(\bm{\theta}, \bm{\phi}; \bm{x}, y)&=q''(z,y)\bm{A}(\bm{\phi}\odot\bm{m}\odot\bm{x})((\bm{A}^T\bm{\theta})\odot(\bm{m}\odot\bm{x}))^T\nonumber+\\
    &\quad+q'(z,y)\operatorname{diag}(\bm{x}\odot\bm{m})\bm{A}^T
\end{align*}

Hence, the Hessian
\begin{align*}
    \nabla^2 J(\bm{\theta}, \bm{\phi}; \bm{x}, y)&=q''(z,y)
    \begin{bmatrix}\bm{A}(\bm{\phi}\odot\bm{m}\odot\bm{x})\\
    (\bm{A}^T\bm{\theta})\odot(\bm{m}\odot\bm{x})
    \end{bmatrix}
    \begin{bmatrix}\bm{A}(\bm{\phi}\odot\bm{m}\odot\bm{x})\\
    (\bm{A}^T\bm{\theta})\odot(\bm{m}\odot\bm{x})
    \end{bmatrix}^T+\nonumber\\
    &\quad+q'(z,y)\begin{bmatrix}
    \bm{0} & \bm{A}\operatorname{diag}(\bm{x}\odot\bm{m})\\
    \operatorname{diag}(\bm{x}\odot\bm{m})\bm{A}^T & \bm{0}
    \end{bmatrix}\\
    &=q''(z,y)
    \begin{bmatrix}\bm{A}\operatorname{diag}(\bm{\phi}\odot\bm{m})\\
    \operatorname{diag}(\bm{A}^T\bm{\theta}\odot\bm{m})
    \end{bmatrix}
        \bm{x}\bm{x}^T
    \begin{bmatrix}\bm{A}\operatorname{diag}(\bm{\phi}\odot\bm{m})\\
    \operatorname{diag}(\bm{A}^T\bm{\theta}\odot\bm{m})
    \end{bmatrix}^T+\nonumber\\
    &\quad+q'(z,y)\begin{bmatrix}
    \bm{0} & \bm{A}\operatorname{diag}(\bm{x}\odot\bm{m})\\
    \operatorname{diag}(\bm{x}\odot\bm{m})\bm{A}^T & \bm{0}
    \end{bmatrix}
\end{align*}

When we optimize a square loss, $q''(z,y)=2$ and $q'(z,y)=2(z-y)$. Thus,
\begin{align*}
    \nabla^2 J(\bm{\theta}, \bm{\phi}; \bm{x}, y)&=2
    \begin{bmatrix}\bm{A}\operatorname{diag}(\bm{\phi}\odot\bm{m})\\
    \operatorname{diag}(\bm{A}^T\bm{\theta}\odot\bm{m})
    \end{bmatrix}
        \bm{x}\bm{x}^T
    \begin{bmatrix}\bm{A}\operatorname{diag}(\bm{\phi}\odot\bm{m})\\
    \operatorname{diag}(\bm{A}^T\bm{\theta}\odot\bm{m})
    \end{bmatrix}^T\nonumber\\
    &\quad+\underbrace{2(z-y)\begin{bmatrix}
    \bm{0} & \bm{A}\operatorname{diag}(\bm{x}\odot\bm{m})\\
    \operatorname{diag}(\bm{x}\odot\bm{m})\bm{A}^T & \bm{0}
    \end{bmatrix}}_{\bm{R}}.
\end{align*}

Let $\bm{e}_{\ell}\in\mathbb{R}^D$ and $\bm{e}'_{\ell}\in\mathbb{R}^M$ be the $\ell$th canonical basis vectors of $\mathbb{R}^D$ and $\mathbb{R}^M$, respectively. Taking the expectation over the data we get
\begin{align*}
    \mathbb{E}_{(\bm{x},y)}\nabla^2 J(\bm{\theta}, \bm{\phi}; \bm{x}, y)&=2
    \begin{bmatrix}\bm{A}\operatorname{diag}(\bm{\phi}\odot\bm{m})\\
    \operatorname{diag}(\bm{A}^T\bm{\theta}\odot\bm{m})
    \end{bmatrix}
    (\epsilon^2\operatorname{diag}(\bm{e}_\ell)+\sigma^2\bm{I}_D) \begin{bmatrix}\bm{A}\operatorname{diag}(\bm{\phi}\odot\bm{m})\\
    \operatorname{diag}(\bm{A}^T\bm{\theta}\odot\bm{m})
    \end{bmatrix}^T\nonumber\\
    &\quad+\mathbb{E}_{(\bm{x},y)}\bm{R}.
\end{align*}
Here, the first summand can be decomposed in a signal and a noise component. The signal component is
\begin{align*}
    \bm{S}&=\begin{bmatrix}\bm{A}\operatorname{diag}(\bm{\phi}\odot\bm{m})\\
    \operatorname{diag}(\bm{A}^T\bm{\theta}\odot\bm{m})
    \end{bmatrix}
    \epsilon^2\operatorname{diag}\left(\bm{e}_\ell\right) \begin{bmatrix}\bm{A}\operatorname{diag}(\bm{\phi}\odot\bm{m})\\
    \operatorname{diag}(\bm{A}^T\bm{\theta}\odot\bm{m})
    \end{bmatrix}^T\\
    &=\epsilon^2\begin{bmatrix}\bm{\phi}[\ell]\bm{m}[\ell]\operatorname{diag}\left(\bm{e'}_{\llbracket\ell\rrbracket_M}\right)\\
    \bm{\theta}\llbracket\ell\rrbracket_M\bm{m}[\ell]\operatorname{diag}\left(\bm{e}_\ell\right)
    \end{bmatrix}
    \begin{bmatrix}\bm{A}\operatorname{diag}(\bm{\phi}\odot\bm{m})\\
    \operatorname{diag}(\bm{A}^T\bm{\theta}\odot\bm{m})
    \end{bmatrix}^T=\\
    &=\epsilon^2\bm{m}^2[\ell]\begin{bmatrix} \bm{\phi}^2[\ell]\operatorname{diag}\left(\bm{e'}_{\llbracket\ell\rrbracket_M}\right) & \bm{\theta}\llbracket\ell\rrbracket_M \bm{\phi}[\ell]\bm{e}'_{\llbracket \ell\rrbracket_M}\bm{e}^T_{\ell}\\
    \bm{\theta}\llbracket\ell\rrbracket_M \bm{\phi}[\ell]\bm{e}_{\ell}\bm{e}'^T_{\llbracket \ell\rrbracket_M}& \bm{\theta}^2\llbracket\ell\rrbracket_M\operatorname{diag}\left(\bm{e}_\ell\right)\end{bmatrix}.
\end{align*}

The noise component is
\begin{align*}
    \bm{W}&=\sigma^2
    \begin{bmatrix}\bm{A}\operatorname{diag}(\bm{\phi}\odot\bm{m})\\
    \operatorname{diag}(\bm{A}^T\bm{\theta}\odot\bm{m})
    \end{bmatrix}
    \begin{bmatrix}\bm{A}\operatorname{diag}(\bm{\phi}\odot\bm{m})\\
    \operatorname{diag}(\bm{A}^T\bm{\theta}\odot\bm{m})
    \end{bmatrix}^T=\\
    &=\sigma^2
    \begin{bmatrix}\bm{A}\operatorname{diag}(\bm{\phi}^2\odot\bm{m}^2) & \bm{A}\operatorname{diag}(\bm{\phi}\odot\bm{m})(\operatorname{diag}(\bm{A}^T\bm{\theta}\odot\bm{m}))^T\\
    (\operatorname{diag}(\bm{A}^T\bm{\theta}\odot\bm{m}))(\bm{A}\operatorname{diag}(\bm{\phi}\odot\bm{m}))^T & \operatorname{diag}(\bm{A}^T\bm{\theta}^2\odot\bm{m}^2)
    \end{bmatrix}
\end{align*}
Taking the expectation over the parameters
\begin{align*}
    \mathbb{E}_{\bm{\theta},\bm{\phi}}\bm{S}&=\epsilon^2\bm{m}^2[\ell]\begin{bmatrix} \sigma^2_{\bm{\phi}}\operatorname{diag}\left(\bm{e'}_{\llbracket\ell\rrbracket_M}\right) & \bm{0}\\
    \bm{0}& \sigma^2_{\bm{\theta}}\operatorname{diag}\left(\bm{e}_{\ell}\right)
    \end{bmatrix}\\
    \mathbb{E}_{\bm{\theta},\bm{\phi}}\bm{W}&=\sigma^2\begin{bmatrix}\bm{A}\sigma^2_{\bm{\phi}}\operatorname{diag}(\bm{m}^2) & \bm{0}\\
    \bm{0} & \sigma^2_{\bm{\theta}}\operatorname{diag}(\bm{m}^2)
    \end{bmatrix},
\end{align*}
and because $\mathbb{E}_{\bm{\theta}} z=\mathbb{E}_{\bm{\phi}} z=0$, and $\mathbb{E}y=0$, then $\mathbb{E}\bm{R}=\bm{0}$.

Overall, we see that
\begin{equation*}
    \mathbb{E}\nabla^2J(\bm{\theta}, \bm{\phi};\bm{x}, y)=\begin{bmatrix}\bm{H}_{\bm{\phi}} & \bm{0}\\
    \bm{0} & \bm{H}_{\bm{\theta}}\end{bmatrix}
\end{equation*}
with
\begin{align*}
    \bm{H}_{\bm{\phi}}&=2\epsilon^2\bm{m}^2[\ell]\sigma^2_{\bm{\theta}}\operatorname{diag}\left(\bm{e}_{\ell}\right)+2\sigma^2\sigma^2_{\bm{\theta}}\operatorname{diag}(\bm{m}^2)\\
    \bm{H}_{\bm{\theta}}&=2\epsilon^2\bm{m}^2[\ell]\sigma^2_{\bm{\phi}}\operatorname{diag}\left(\bm{e'}_{\llbracket\ell\rrbracket_M}\right)+2\sigma^2\sigma^2_{\bm{\phi}}\bm{A}\operatorname{diag}(\bm{m}^2)
\end{align*}
\end{proof}

\subsection{Proof of Lemma~\ref{lemma:bound}}

We prove Lemma~\ref{lemma:bound} under a slightly more general setting than in the text. 
\begin{lemma*}
Let $g:[0,\infty]\rightarrow[0,\infty]$ be an increasing function with polynomially bounded first-order derivative, i.e., $|g'(t)|\leq \omega_n(|t|)$, where $\omega_n:\mathbb{R}\rightarrow\mathbb{R}$ is an $n$-order polynomial.

The expected value of $\|\nabla_{\bm{\theta}}q_{\bm{\theta}}(\bm{v})\|$ is bounded by
\begin{align*}
  \mathbb{E} \|\nabla_{\bm{\theta}}q_{\bm{\theta}}(\bm{v})\|&\leq \sqrt{\mathbb{E}\omega^2_n\left(\left|\bm{v}^T\nabla_{\bm{x}}f_{\bm{\theta}}(\bm{x})\right|\right)}\sqrt{\mathbb{E}\|\nabla^2_{\bm{\theta},\bm{x}}f_{\bm{\theta}}(\bm{x})\bm{v}\|^2}
\end{align*}
\end{lemma*}

\begin{proof}
Using the polynomial bound on the derivative of $g$ and using Cauchy-Schwarz inequality we can bound the expected norm of $\nabla_{\bm{\theta}}q_{\theta}(\bm{x})$ as
\begin{align*}
  \mathbb{E} \|\nabla_{\bm{\theta}}q_{\bm{\theta}}(\bm{v})\|&= \mathbb{E}| g'\left(\left|\bm{v}^T\nabla_{\bm{x}}f_{\bm{\theta}}(\bm{x})\right|\right)|\|\nabla^2_{\bm{\theta},\bm{x}}f_{\bm{\theta}}(\bm{x})\bm{v}\|\\
  &\leq \mathbb{E} \omega_n\left(\left|\bm{v}^T\nabla_{\bm{x}}f_{\bm{\theta}}(\bm{x})\right|\right)\|\nabla^2_{\bm{\theta},\bm{x}}f_{\bm{\theta}}(\bm{x})\bm{v}\|\\
  &\leq \sqrt{\mathbb{E}\omega^2_n\left(\left|\bm{v}^T\nabla_{\bm{x}}f_{\bm{\theta}}(\bm{x})\right|\right)}\sqrt{\mathbb{E}\|\nabla^2_{\bm{\theta},\bm{x}}f_{\bm{\theta}}(\bm{x})\bm{v}\|^2}
\end{align*}
\end{proof}

We see that this bound depends on the spectral decomposition of the moments of $\nabla_{\bm{x}}f_{\bm{\theta}}(\bm{x})$ up to order $2n$, e.g., its covariance $\mathbb{E}\nabla_{\bm{x}}f_{\bm{\theta}}(\bm{x})\nabla^T_{\bm{x}}f_{\bm{\theta}}(\bm{x})$, and the expected right singular vectors of the mixed second derivative $\nabla^2_{\bm{\theta},\bm{x}}f_{\bm{\theta}}(\bm{x})$. In the case of the text $\omega_n(t)=\alpha t+\beta$. Hence, $n=1$ and the bound only depends on the gradient covariance and second derivative.

\newpage
\section{Analytic NAD examples}
\subsection{Proofs for linear model of pooling}

We first prove the expressions for the example in the text.
\begin{example}
Let $\bm{\phi}\sim\mathcal{N}(\bm{0}, \sigma^2_{\bm{\phi}}\bm{I}_D)$ and $\bm{\theta}\sim\mathcal{N}(\bm{0}, \sigma^2_{\bm{\theta}}\bm{I}_M)$, the covariance of the input gradient of the linear model of pooling is
\begin{equation*}
    \mathbb{E}\nabla_{\bm{x}}f_{\bm{\theta},\bm{\phi}}(\bm{x})\nabla^T_{\bm{x}}f_{\bm{\theta},\bm{\phi}}(\bm{x})=\sigma_{\bm{\phi}}^2\sigma_{\bm{\theta}}^2\operatorname{diag}\left(\bm{m}^2\right),
\end{equation*}
and its eigenvectors are the canonical basis elements of $\mathbb{R}^D$, sorted by the entries of $\bm{m}^2$. Surprisingly, the expected right singular vectors of its mixed second derivative coincide with these eigenvectors,
\begin{equation*}
    \mathbb{E}\nabla^2_{(\bm{\theta},\bm{\phi}),\bm{x}}f(\bm{x})^T\nabla^2_{(\bm{\theta},\bm{\phi}),\bm{x}}f(\bm{x})=\left(\sigma^2_{\bm{\theta}}+\cfrac{\sigma^2_{\bm{\phi}}}{S}\right)\operatorname{diag}(\bm{m}^2).
\end{equation*}
This result agrees with what was seen in Sec.~\ref{sec:condition} where we found that the NADs of this architecture are also ranked by $\bm{m}^2$.
\end{example}
\begin{proof}
Borrowing the gradient computations from Sec.~\ref{sec:proof_lemma_1},
\begin{align*}
\mathbb{\bm{E}}_{\bm{\theta}}\nabla_{\bm{x}} f_{\bm{\theta}}(\bm{x})\nabla_{\bm{x}} f_{\bm{\theta}}(\bm{x})^T&=\mathbb{\bm{E}}_{\bm{\theta}, \bm{\phi}}[(\bm{A}^T\bm{\theta})\odot(\bm{\phi}\odot\bm{m})][(\bm{\phi}^T\odot\bm{m}^T)\odot(\bm{\theta}^T\bm{A})]\\
&=\mathbb{E}_{\bm{\phi}}\bm{\phi}\bm{\phi}^T\odot \mathbb{E}_{\bm{\theta}}\bm{A}^T\bm{\theta}\bm{\theta}^T\bm{A}\odot \bm{m}\bm{m}^T=\sigma_{\bm{\phi}}^2\sigma_{\bm{\theta}}^2\left(\bm{I}\odot\bm{A}^T\bm{A}\odot \bm{m}\bm{m}^T\right)\\
&=\sigma_{\bm{\phi}}^2\sigma_{\bm{\theta}}^2\operatorname{diag}\left(\bm{m}^2\right).
\end{align*}
Similarly, the mixed second derivatives for this model are
\begin{align*}
    \nabla^2_{\bm{\theta},\bm{x}}f(\bm{x})&=\operatorname{diag}(\bm{\phi}\odot\bm{m})\bm{A}^T\\
    \nabla^2_{\bm{\phi},\bm{x}}f(\bm{x})&=\operatorname{diag}\left((\bm{A}^T\bm{\theta})\odot\bm{m}\right)
\end{align*}
which can be combined in
\begin{align*}
\nabla^2_{(\bm{\theta},\bm{\phi}),\bm{x}}f(\bm{x})=\begin{bmatrix}
\operatorname{diag}\left((\bm{A}^T\bm{\theta})\odot\bm{m}\right)\\
\bm{A}\operatorname{diag}(\bm{\phi}\odot\bm{m})
\end{bmatrix}.
\end{align*}
We can extract its right singular vectors from the eigendecomposition of
\begin{align*}
    \mathbb{E}\nabla^2_{(\bm{\theta},\bm{\phi}),\bm{x}}f(\bm{x})^T\nabla^2_{(\bm{\theta},\bm{\phi}),\bm{x}}f(\bm{x})&=\mathbb{E}\left[\operatorname{diag}\left((\bm{A}^T\bm{\theta})^2\odot\bm{m}^2\right)\right]\\
    &\quad+\mathbb{E}\left[\operatorname{diag}\left(\bm{\phi}\odot\bm{m}\right)\bm{A}^T\bm{A}\operatorname{diag}\left(\bm{\phi}\odot\bm{m}\right)\right]=\\
    &=\sigma^2_{\bm{\theta}}\operatorname{diag}(\bm{m}^2)+\cfrac{\sigma^2_{\bm{\phi}}}{S}\operatorname{diag}(\bm{m}^2)\\
    &=\left(\sigma^2_{\bm{\theta}}+\cfrac{\sigma^2_{\bm{\phi}}}{S}\right)\operatorname{diag}(\bm{m}^2).
\end{align*}
\end{proof}

\subsection{More examples}

We provide a few more examples showing that the gradient covariance can indeed capture the NADs of an architecture.

\begin{example}[Logistic regression]
Let $f_{\bm{\theta}}(\bm{x})=\bm{\theta}^T\bm{x}$ be a single layer neural network, i.e., logistic regression. The gradient covariance of this architecture is
\begin{equation*}
    \mathbb{E}\nabla_{\bm{x}}f_{\bm{\theta}}(\bm{x})\nabla^T_{\bm{x}}f_{\bm{\theta}}(\bm{x})=\sigma^2_{\bm{\theta}}\bm{I}_D.
\end{equation*}
Because the eigendecomposition of $\bm{I}_D$ is isotropic, we can see that the logistic regression has no directional bias.
\end{example}

\begin{example}[Single hidden-layer neural network]\label{ex:single_hidden_layer}
Let $f_{\bm{\theta},\bm{\Phi}}(\bm{x})=\bm{\theta}^T\rho\left(\bm{\Phi}^T\bm{x}\right)$ be a single hidden layer neural network with no bias and a ReLU non-linearity $\rho(\cdot)$. Its gradient covariance is
\begin{align*}
     \mathbb{E}\nabla_{\bm{x}}f_{\bm{\theta},\bm{\Phi}}(\bm{x})\nabla^T_{\bm{x}}f_{\bm{\theta},\bm{\Phi}}(\bm{x})=\cfrac{\sigma^2_{\bm{\theta}}\sigma^2_{\bm{\Phi}}}{2},
\end{align*}
and we see that this architecture has also no directional bias.
\end{example}

\begin{proof}
The gradient of $f_{\bm{\theta},\bm{\Phi}}(\bm{x})$ is $\nabla_{\bm{x}}f_{\bm{\theta},\bm{\Phi}}(\bm{x})=\bm{\Phi}\operatorname{diag}\left(\rho'\left(\bm{\Phi}^T\bm{x}\right)\right)\bm{\theta}$, where the derivative of the ReLU non-linearity is the indicator function $\rho'(\bm{u})=\mathbbm{1}_{\bm{u}\succeq \bm{0}}$.

Hence,
\begin{align*}
    \mathbb{E}\nabla_{\bm{x}}f_{\bm{\theta},\bm{\Phi}}(\bm{x})\nabla^T_{\bm{x}}f_{\bm{\theta},\bm{\Phi}}(\bm{x})&=\mathbb{E}\left[\bm{\Phi}\operatorname{diag}\left(\rho'\left(\bm{\Phi}^T\bm{x}\right)\right)\bm{\theta}\bm{\theta}^T\operatorname{diag}\left(\rho'\left(\bm{\Phi}^T\bm{x}\right)\right)\bm{\Phi}^T\right]=\\
    &=\sigma^2_{\bm{\theta}}\mathbb{E}\left[\bm{\Phi}\operatorname{diag}\left(\rho'\left(\bm{\Phi}^T\bm{x}\right)\right)\operatorname{diag}\left(\rho'\left(\bm{\Phi}^T\bm{x}\right)\right)\bm{\Phi}^T\right]=\\
    &=\sigma^2_{\bm{\theta}}\mathbb{E}\left[\bm{\Phi}\operatorname{diag}\left(\mathbbm{1}_{\bm{\Phi}^T\bm{x}\succeq \bm{0}}\right)\bm{\Phi}^T\right]
\end{align*}
This expectation can be computed analytically. In particular note that
\begin{align*}
    \mathbb{E}\left[\bm{\Phi}\operatorname{diag}\left(\mathbbm{1}_{\bm{\Phi}^T\bm{x}\succeq \bm{0}}\right)\bm{\Phi}^T\right][i,j]=\sum_{k=1}^D \mathbb{E}\left[\bm{\Phi}[i,k]\bm{\Phi}[j,k]\mathbbm{1}_{\Phi[i,:]^T\bm{x}\geq 0}\right].
\end{align*}
Therefore, if $i\neq j$
\begin{equation*}
    \mathbb{E}\left[\bm{\Phi}[i,k]\bm{\Phi}[j,k]\mathbbm{1}_{\Phi[i,:]^T\bm{x}\geq 0}\right]=\mathbb{E}_{\bm{\Phi}[i,k]}\left[\mathbb{E}_{\bm{\Phi}[j,k]}\left[\bm{\Phi}[i,k]\bm{\Phi}[j,k]\mathbbm{1}_{\Phi[i,:]^T\bm{x}\geq 0}\right|\bm{\Phi}[i,k]\right]=0.
\end{equation*}
On the other hand, when $i=j$,
\begin{align*}
     \mathbb{E}\left[\bm{\Phi}\operatorname{diag}\left(\mathbbm{1}_{\bm{\Phi}^T\bm{x}\succeq \bm{0}}\right)\bm{\Phi}^T\right][i,i]&=\sum_{k=1}^D\mathbb{E}\left[\bm{\Phi}^2[i,k]\mathbbm{1}_{\Phi[i,:]^T\bm{x}\geq 0}\right]\\
     &=\mathbb{E}\left[\|\bm{\Phi}[i,:]\|^2\mathbbm{1}_{\Phi[i,:]^T\bm{x}\geq 0}\right].
\end{align*}
Let $p(\bm{w})$ denote the probability density function of a Gaussian random vector  $\bm{w}\sim\mathcal{N}(\bm{0}, \sigma^2\bm{I})$ and $\bm{U}\in\mathrm{SO}(D)$ and orthonormal matrix such that $\bm{x}'=\bm{U}^T\bm{x}$ with $\bm{x}'[1]=\|\bm{x}\|$ and $\bm{x}'[i]=0$ for $i=2,\dots, D$. Then,
\begin{align*}
    \langle\bm{w},\bm{x}\rangle\geq 0\Leftrightarrow \langle\bm{U}\bm{w},\bm{x}\rangle\geq 0 \rangle \Leftrightarrow \langle\bm{w},\bm{U}^T\bm{x}\rangle\geq 0\Leftrightarrow \bm{w}[1]\|\bm{x}\|_2\geq 0 \Leftrightarrow \bm{w}[1]\geq 0.
\end{align*}
Using this equivalence, we can compute the expectation
\begin{align*}
    \mathbb{E}\left[\|\bm{w}\|^2\mathbbm{1}_{\bm{w}^T\bm{x}\geq 0}\right]&=\int_{\mathbb{R}^D}\mathbbm{1}_{\bm{w}^T\bm{x}\geq 0}\|\bm{w}\|^2 p(\bm{w})d\bm{w}=\int_{\mathbb{R}^D}\mathbbm{1}_{\bm{w}[1]\geq 0}\|\bm{w}\|^2 p(\bm{w})d\bm{w}=\\
    &=\int_{\mathbb{R}^D}\mathbbm{1}_{\bm{w}[1]\geq 0}\bm{w}^2[1] p(\bm{w})d\bm{w}+\sum_{i=2}^D\int_{\mathbb{R}^D}\mathbbm{1}_{\bm{w}[1]\geq 0}\bm{w}^2[i] p(\bm{w})d\bm{w}=\\
    &=\int_{0}^{+\infty}\bm{w}^2[1] \cfrac{1}{\sqrt{2\pi\sigma^2}}\,e^{-\frac{\bm{w}^2[1]}{2\sigma^2}}d\bm{w}[1]+\\
    &\quad+\cfrac{D-1}{2}\int_{-\infty}^{+\infty}\bm{w}^2[2] \cfrac{1}{\sqrt{2\pi\sigma^2}}\,e^{-\frac{\bm{w}^2[2]}{2\sigma^2}}d\bm{w}[2]=\\
    &=\cfrac{1}{2}\sigma^2+\cfrac{D-1}{2}\sigma^2=\cfrac{D}{2}\sigma^2.
\end{align*}

Plugging this into the expressions of the gradient covariance we get
\begin{align*}
    \mathbb{E}\nabla_{\bm{x}}f_{\bm{\theta},\bm{\Phi}}(\bm{x})\nabla^T_{\bm{x}}f_{\bm{\theta},\bm{\Phi}}(\bm{x})&=\sigma^2_{\bm{\theta}}\mathbb{E}\left[\bm{\Phi}\operatorname{diag}\left(\mathbbm{1}_{\bm{\Phi}^T\bm{x}\succeq \bm{0}}\right)\bm{\Phi}^T\right]=\\
    &=\sigma^2_{\bm{\theta}}\mathbb{E}\left[\|\bm{\Phi}[i,:]\|^2\mathbbm{1}_{\Phi[i,:]^T\bm{x}\geq 0}\right]\bm{I}_D=\\
    &=\cfrac{D}{2}\,\sigma^2_{\bm{\theta}}\sigma^2_{\bm{\Phi}}\bm{I}_D.
\end{align*}
\end{proof}

\begin{example}[Non-linear model of pooling]
Let $f_{\bm{\theta},\bm{\phi}}(\bm{x})=\bm{\theta}^T\bm{A}(\bm{m}\odot\rho(\bm{\phi}\odot\bm{v}))$ with $\nabla_{\bm{x}}f_{\bm{\theta},\bm{\phi}}(\bm{x})=(\bm{A}^T\bm{\theta})\odot(\rho'(\bm{\phi}\odot\bm{x})\odot\bm{\phi}\odot\bm{m})$. Then, 
\begin{align*}
    \mathbb{E}\nabla_{\bm{x}}f_{\bm{\theta},\bm{\phi}}(\bm{x})\nabla^T_{\bm{x}}f_{\bm{\theta},\bm{\phi}}(\bm{x})&=\sigma_{\bm{\phi}}^2\sigma_{\bm{\theta}}^2\left(\bm{A}^T\bm{A}\odot \bm{m}\bm{m}^T\odot\bm{\Xi}(\bm{x})\right),
\end{align*}
where $\bm{\Xi}(\bm{x})\in\mathbb{R}^{D\times D}$ is a matrix that depends on the input vector $\bm{x}$ and can be computed in closed form. 

In particular, if the distribution of $\bm{x}$ is symmetric around $\bm{0}$, then $\mathbb{E}\,\bm{\Xi}(\bm{x})=\bm{I}_D$ and the average gradient covariance with respect to the input would be identical to that of the linear model of pooling.
\end{example}

\begin{proof}

Expanding the covariance definition
\begin{align*}
    \mathbb{E}\nabla_{\bm{x}}f_{\bm{\theta},\bm{\phi}}(\bm{x})\nabla^T_{\bm{x}}f_{\bm{\theta},\bm{\phi}}(\bm{x})&=\mathbb{\bm{E}}[(\bm{A}^T\bm{\theta})\odot(\rho'(\bm{\phi}\odot\bm{x})\odot\bm{\phi}\odot\bm{m})]\\
    &\qquad[(\rho'(\bm{\phi}^T\odot\bm{x}^T)\odot\bm{\phi}^T\odot\bm{m}^T)\odot(\bm{\theta}^T\bm{A})]=\\
    &=\mathbb{E}[(\rho'(\bm{\phi}\odot\bm{x})\odot\bm{\phi}][(\rho'(\bm{\phi}^T\odot\bm{x}^T)\odot\bm{\phi}^T]\odot \mathbb{E}\bm{A}^T\bm{\theta}\bm{\theta}^T\bm{A}\odot \bm{m}\bm{m}^T.
\end{align*}

We can see that the only difference with respect to the linear model case is the first expectation. Let $\bm{\Xi}(\bm{x})\in\R^{D\times D}$ be the matrix with entries
\begin{align*}
    \bm{\Xi}[i,j]&=\mathbb{E}[(\rho'(\bm{\phi}\odot\bm{x})\odot\bm{\phi}][(\rho'(\bm{\phi}^T\odot\bm{x}^T)\odot\bm{\phi}^T][i,j]\\
    &=\mathbb{E}[\bm{\phi}[i]\bm{\phi}[j]\mathbbm{1}_{\bm{\phi}[i]\bm{x}[i]\geq 0}\mathbbm{1}_{\bm{\phi}[j]\bm{x}[]j]\geq 0}]=\begin{cases}
    \mathbb{E}[\bm{\phi}[i]\mathbbm{1}_{\bm{\phi}[i]\bm{x}[i]\geq 0}]\mathbb{E}[\bm{\phi}[j]\mathbbm{1}_{\bm{\phi}[j]\bm{x}[j]\geq 0}] & i\neq j\\
    \mathbb{E}[\bm{\phi}^2[j]\mathbbm{1}_{\bm{\phi}[j]\bm{x}[j]\geq 0}] & i=j
    \end{cases}
\end{align*}

Depending on $\bm{x}$ the expectation $\mathbb{E}[\bm{\phi}[j]\mathbbm{1}_{\bm{\phi}[j]\bm{x}[j]\geq 0}]$ takes different values:
\begin{equation*}
    \mathbb{E}[\bm{\phi}[j]\mathbbm{1}_{\bm{\phi}[j]\bm{x}[j]\geq 0}]=\begin{cases}
    \cfrac{\sigma_{\bm{\phi}}\sqrt{2}}{2\sqrt{\pi}} & \bm{x}[j]>0\\
    -\cfrac{\sigma_{\bm{\phi}}\sqrt{2}}{2\sqrt{\pi}} & \bm{x}[j]<0\\
    0 & \bm{x}[i]=0
    \end{cases}
\end{equation*}
Similarly
\begin{equation*}
    \mathbb{E}[\bm{\phi}^2[j]\mathbbm{1}_{\bm{\phi}[j]\bm{x}[j]\geq 0}]=\begin{cases}
    \cfrac{\sigma^2_{\bm{\phi}}}{2}\left(1-\cfrac{2}{\pi}\right) & \bm{x}[j]\neq0\\
    \sigma^2_{\bm{\phi}} & \bm{x}[j]=0
    \end{cases}
\end{equation*}

Then the covariance depending on $\bm{x}$ becomes,
\begin{align*}
    \mathbb{E}\nabla_{\bm{x}}f_{\bm{\theta},\bm{\phi}}(\bm{x})\nabla^T_{\bm{x}}f_{\bm{\theta},\bm{\phi}}(\bm{x})=\sigma_{\bm{\phi}}^2\sigma_{\bm{\theta}}^2\left(\bm{A}^T\bm{A}\odot \bm{m}\bm{m}^T\odot\bm{\Xi}(\bm{x})\right).
\end{align*}

\end{proof}

\newpage
\section{NADs of CNNs}

As highlighted in Sec.~\ref{sec:nads_cnns}, we can use two algorithms to identify the NADs of an architecture without training. Surprisingly, both algorithms yield very similar results, but the algorithm based on the eigendecomposition of the gradient covariance is numerically much more stable. Indeed, for most randomly initialized networks, the norm of the second derivative with respect to the weights and input is very small, rendering the numerical singular value decomposition of the second derivative very unstable. Meanwhile, the gradient covariance only requires information about first order gradients and these are orders of magnitudes larger than the second derivatives. For this reason, in all our experiments we used the eigenvectors of the gradient covariance as approximations of the NADs of a given architecture.

We provide now the implementation details of both algorithms, as well as some examples of NADs identified with both methods.

\subsection{NADs obtained through the eigendecomposition of the gradient covariance}\label{sec:NAD_gradient}

Algorithm~\ref{alg:gradient} describes the steps required to identify the NADs of an architecture using its input gradient covariance. As we can see, this procedure amounts to sampling $T$ architectures from its weight initialization distribution, computing its input gradient at an arbitrary input point $\bm{x}$, and performing a Principal Component Analysis on the gradient samples.
\begin{algorithm}[ht!]
\caption{NAD discovery through gradient covariance}\label{alg:gradient}
\begin{algorithmic}[1]
\Require Network architecture $f_{\bm{\theta}}$, parameter distribution $\bm{\Theta}$, evaluation sample $\bm{x}$, number of Monte-Carlo samples $T$, and finite-difference scale $h$.
\State $\mathcal{G}\gets\varnothing$\Comment{Gradient samples}
\For{$t=1,\dots, T$}
    \State Draw $\bm{\theta}\sim\bm{\Theta}$
    \State $\bar{\nabla}_{\bm{x}}f_{\bm{\theta}}(\bm{x})\gets \bm{0}$
    \For{$i=1,\dots,D$}
        \State $\bar{\nabla}_{\bm{x}}f_{\bm{\theta}}(\bm{x})[i]\gets \cfrac{f_{\bm{\theta}}(\bm{x}+h\bm{e}_i)-f_{\bm{\theta}}(\bm{x}-h\bm{e}_i)}{2h}$\Comment{Compute finite difference gradient}
    \EndFor
    \State $\mathcal{G}\gets \mathcal{G}\cup \bar{\nabla}_{\bm{x}}f_{\bm{\theta}}(\bm{x})$
\EndFor
\State $\{(\bm{u}_i,\lambda_i)\}_{i=1}^D\gets \operatorname{PCA}(\mathcal{G})$\Comment{Perform Principal Component Analysis}
\State \Return $\{\bm{u}_i\}_{i=1}^D$
\end{algorithmic}
\end{algorithm}

In practice, we found out that using finite differences with a scale of $h=100$ to approximate the gradients instead of backpropagation was necessary to obtain meaningful results. We believe the reason for this is that the finite differences allow to capture a coarser scale of the function geometry and hide the effect of higher order terms, as they do not rely on very local fluctuations of the input geometry. We leave for future research the understanding of this phenomenon.

We now show some additional examples of NADs obtained using Algorithm~\ref{alg:gradient} on a LeNet, VGG-11, ResNet-18 and DenseNet121.

\newpage
\subsubsection{LeNet}
\begin{figure}[ht!]
    \centering
    \includegraphics[width=\textwidth]{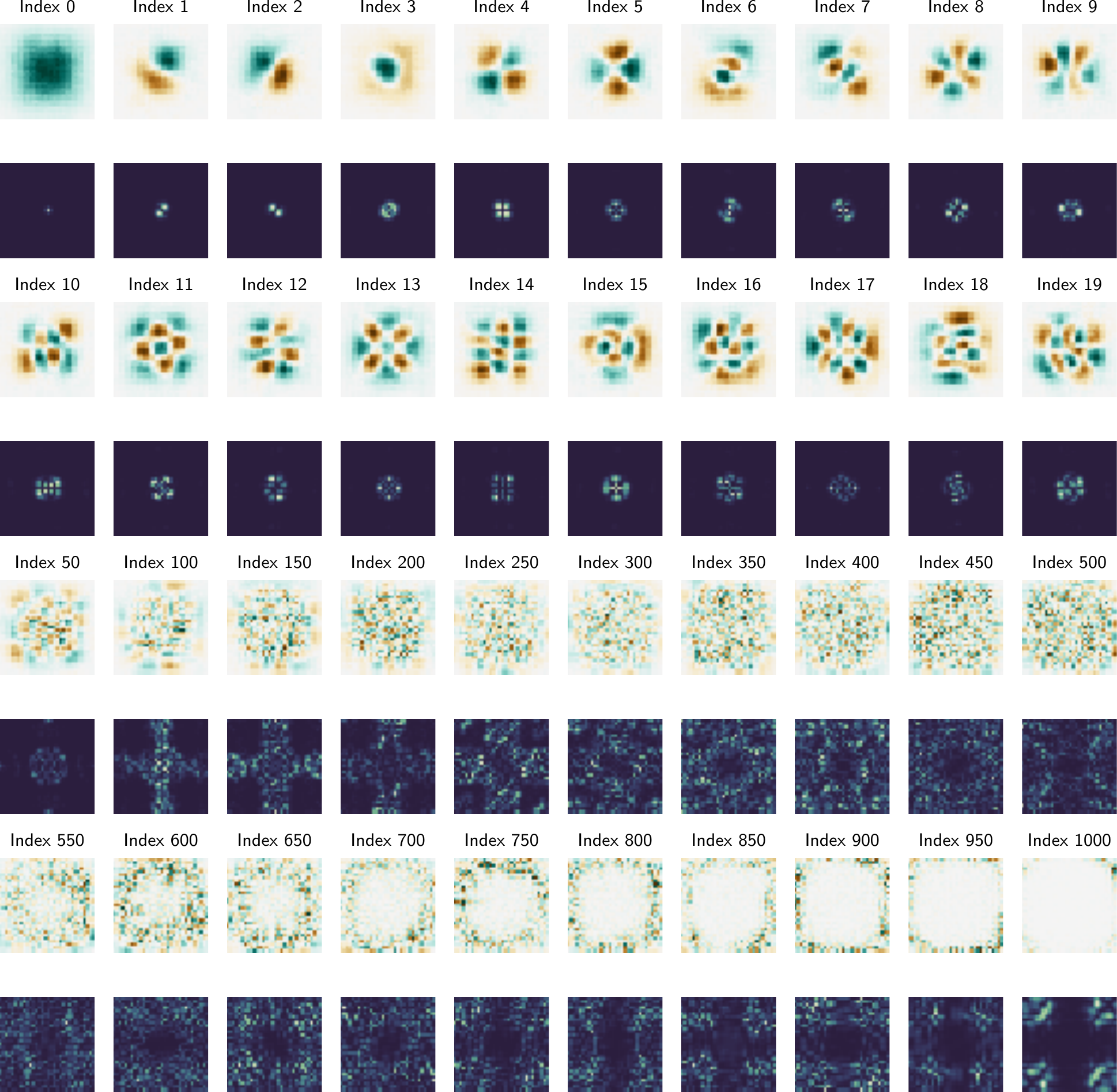}
    \caption{NADs of LeNet obtained through eigendecomposition of gradient covariance}
    \label{fig:lenet_nads_first}
\end{figure}
\newpage

\subsubsection{VGG11}
\begin{figure}[ht!]
    \centering
    \includegraphics[width=\textwidth]{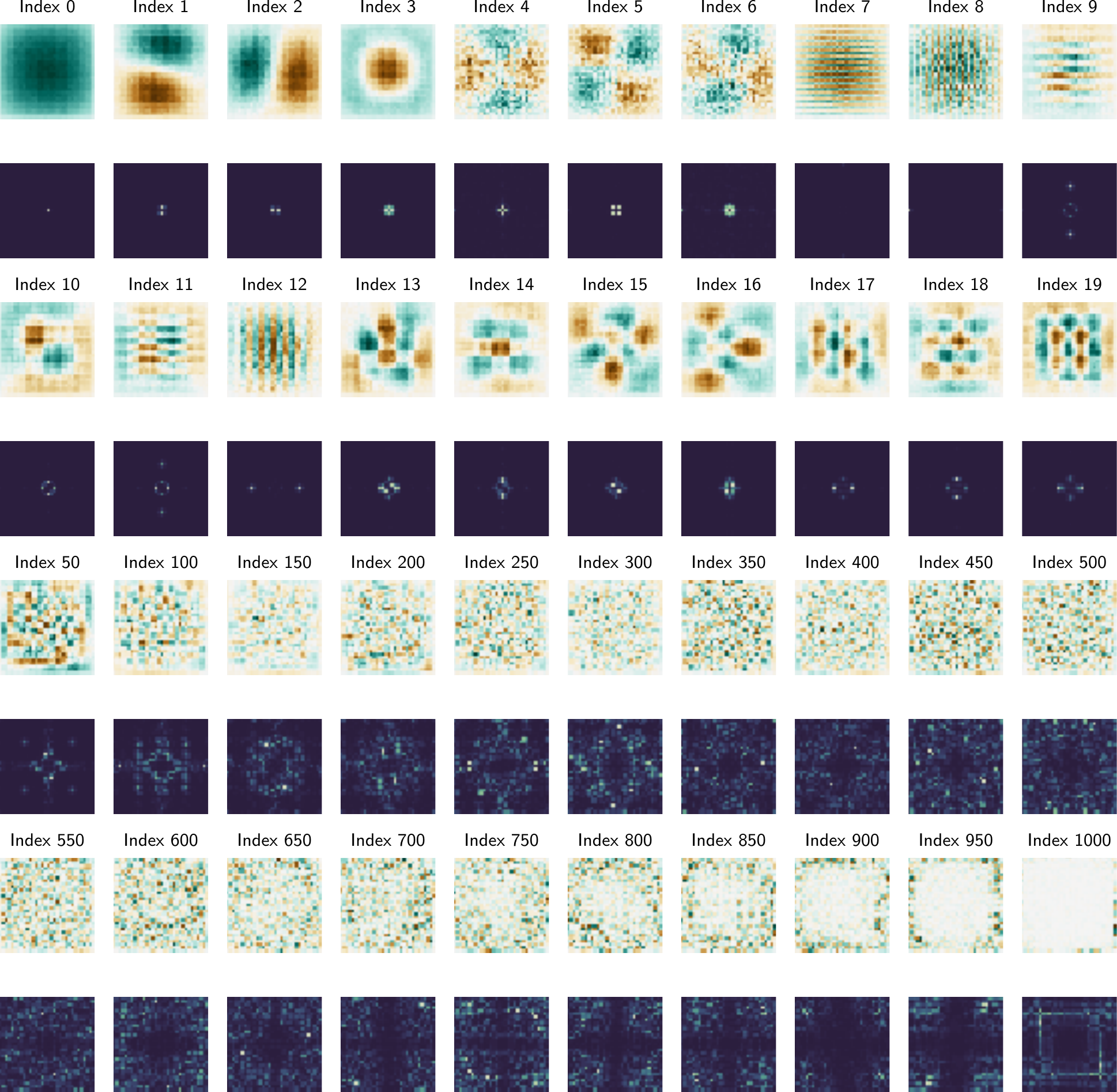}
    \caption{NADs of VGG16 obtained through eigendecomposition of gradient covariance}
    \label{fig:vgg_nads_first}
\end{figure}
\newpage

\subsubsection{ResNet-18}
\begin{figure}[ht!]
    \centering
    \includegraphics[width=\textwidth]{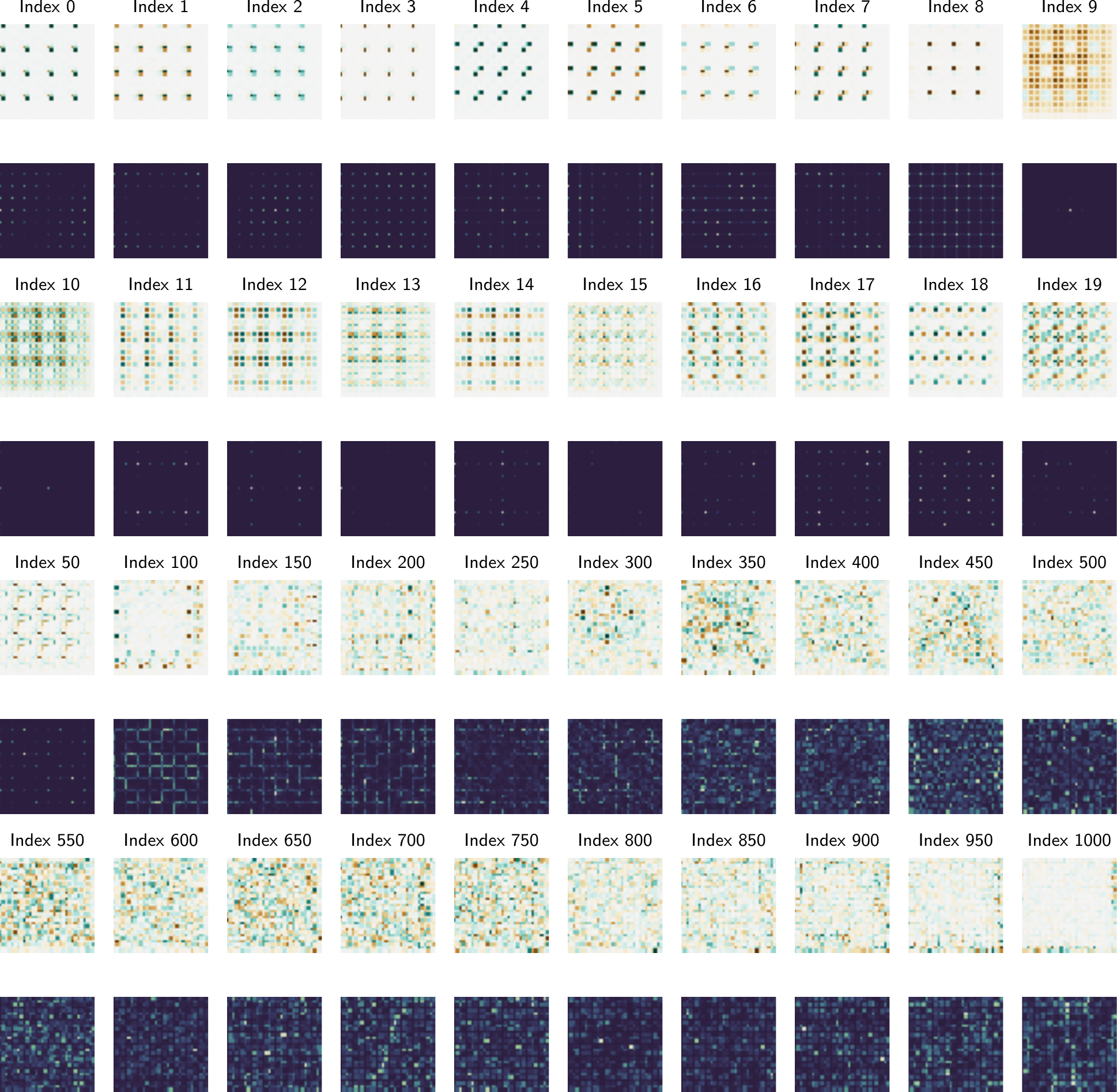}
    \caption{NADs of ResNet-18 obtained through eigendecomposition of gradient covariance}
    \label{fig:resnet_nads_first}
\end{figure}
\newpage

\subsubsection{DenseNet-121}
\begin{figure}[ht!]
    \centering
    \includegraphics[width=\textwidth]{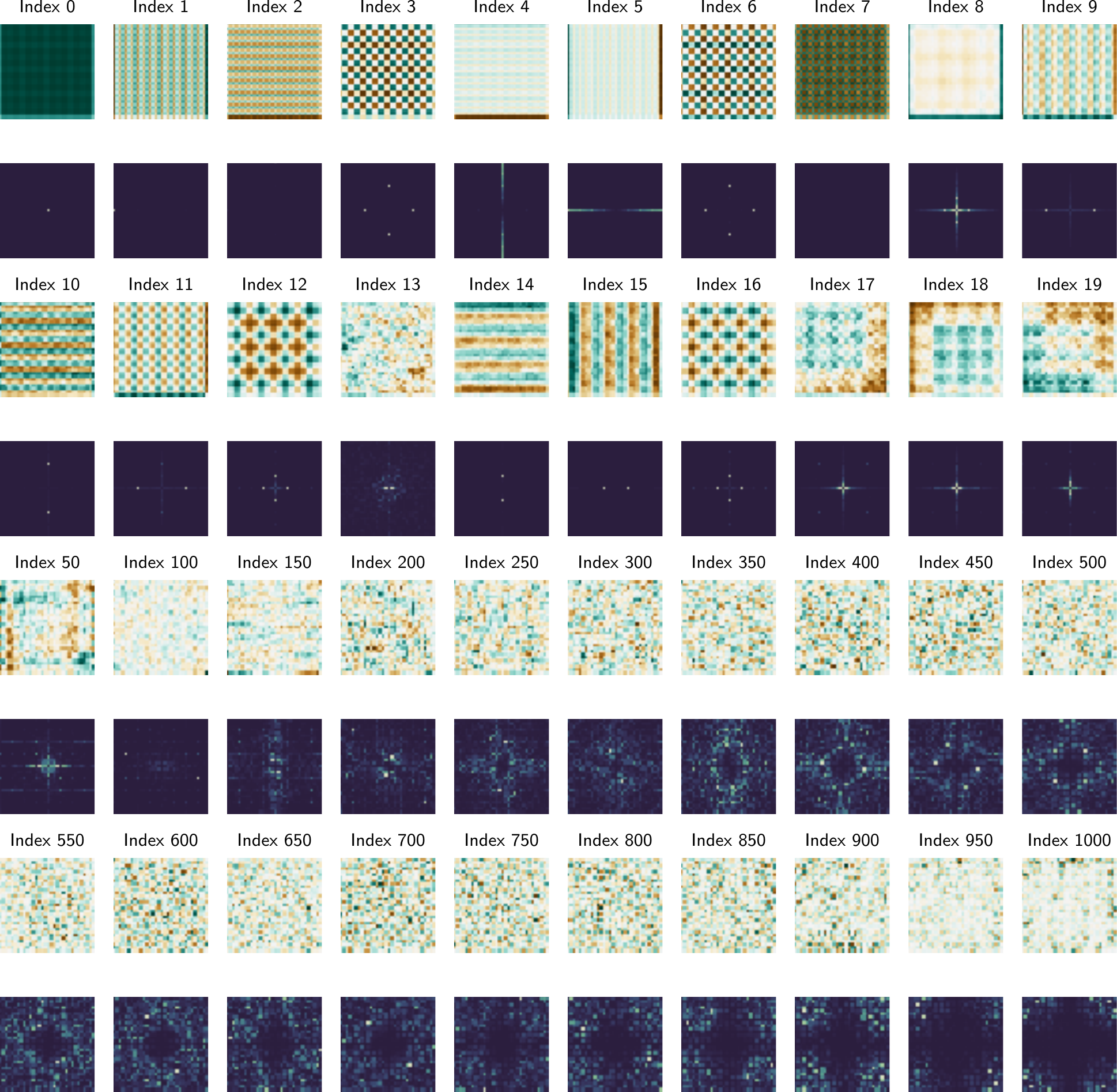}
    \caption{NADs of DenseNet-121 obtained through eigendecomposition of gradient covariance}
    \label{fig:densenet_nads_first}
\end{figure}
\newpage

\subsection{NADs obtained through the SVD of the mixed second derivative}\label{sec:NAD_second_derivative}
The second way we can identify the NADs without training is using the expected right singular vectors of the mixed second derivative, $\nabla^2_{\bm{x},\bm{\theta}}f_{\bm{\theta}}(\bm{x})$. However, note that the mixed second derivative has a number of entries equal to the product of the weight and input dimensionalities, which can amount to more than a trillion elements. This makes it impossible to store this object in any common computational platform, and hence we can only estimate its singular vectors using power iteration methods~\cite{matrix_computation}. Specifically, these methods estimate the spectral decomposition of a linear operator by sequentially alternating between the application of the linear operator on a vector and its adjoint. 

Consequently, we just need an efficient way to compute $\nabla^2_{\bm{x},\bm{\theta}}f_{\bm{\theta}(\bm{x})}\bm{v}$ and $\bm{v'}^T\nabla^2_{\bm{x},\bm{\theta}}f_{\bm{\theta}(\bm{x})}$ for any $\bm{v}$ and $\bm{v'}$ to be able to compute the SVD. Algorithm~\ref{alg:second_derivative} details these procedures. As we can see, in our algorithms we use a finite difference approximation to compute the directional input derivative of $\nabla_{\bm{\theta}}f_{\bm{\theta}}(\bm{x})$. Again, this helps for stability of the results.

\begin{algorithm}
\caption{NAD discovery through mixed second derivative}\label{alg:second_derivative}
\begin{algorithmic}[1]
\Require Network architecture $f_{\bm{\theta}}$, parameter distribution $\bm{\Theta}$, evaluation sample $\bm{x}$, number of Monte-Carlo samples $T$, and finite-difference scale $h$.
\Procedure{DVP}{$\mathcal{F}$, $\bm{v}$} \Comment{Computes $\nabla^2_{\bm{x},\bm{\theta}}f_{\bm{\theta}(\bm{x})}\bm{v}$}
\For{$f_{\bm{\theta}}\in\mathcal{F}$}
\State $d\gets 0$
\State $d\gets d + \cfrac{\nabla_{\bm{\theta}}f_{\bm{\theta}}(\bm{x}+h\bm{v})-\nabla_{\bm{\theta}}f_{\bm{\theta}}(\bm{x}-h\bm{v})}{2h}$
\EndFor
\State \Return $d / T$
\EndProcedure
\Statex
\Procedure{ADVP}{$\mathcal{F}$, $\bm{v'}$} \Comment{Computes $\bm{v'}^T\nabla^2_{\bm{x},\bm{\theta}}f_{\bm{\theta}(\bm{x})}$}
\For{$f_{\bm{\theta}}\in\mathcal{F}$}
\State $d\gets 0$
\State $d\gets d + \nabla_{\bm{x}}\left(\bm{v}'^T\nabla_{\bm{\theta}}f_{\bm{\theta}}(\bm{x})\right)$
\EndFor
\State \Return $d / T$
\EndProcedure
\Statex

\State $\mathcal{F}\gets\varnothing$ \Comment{Function samples}
\For{$t=1,\dots, T$}
    \State Draw $\bm{\theta}\sim\bm{\Theta}$
    \State $\mathcal{F}\gets \mathcal{F}\cup f_{\bm{\theta}}$
\EndFor
\Statex
\State $\{(\bm{u}_i,\sigma_i)\}\gets \operatorname{PowerIteration}(\operatorname{DVP}, \operatorname{ADVP})$\Comment{SVD through power iterations}
\State \Return $\{\bm{u}_i\}_{i=1}^D$
\end{algorithmic}
\end{algorithm}

In the next figures, we show the results of the application of these algorithm to a LeNet, VGG-10 and ResNet-18. However, due to the high computational complexity of Algorithm~\ref{alg:second_derivative} on large networks, we do not show them for the larger DenseNet-121. At this stage, it is important to highlight that the results of Algorithm~\ref{alg:second_derivative} are much noisier than those of Algorithm~\ref{alg:gradient} (as seen in the resulting NADs depicted in Sec.~\ref{sec:NAD_gradient} and Sec.~\ref{sec:NAD_second_derivative}). We believe this is due to the bad conditioning of Algorithm~\ref{alg:second_derivative} due to the small magnitude of the second derivatives and the use of a power iteration method in Algorithm~\ref{alg:second_derivative} with respect to the exact eigendecomposition in Algorithm~\ref{alg:gradient}. Nevertheless, looking at the shape (especially in the spectral domain) of the first few NADs obtained with both algorithms we can see that they are indeed very aligned.

\newpage
\subsubsection{LeNet}
\begin{figure}[ht!]
    \centering
    \includegraphics[width=\textwidth]{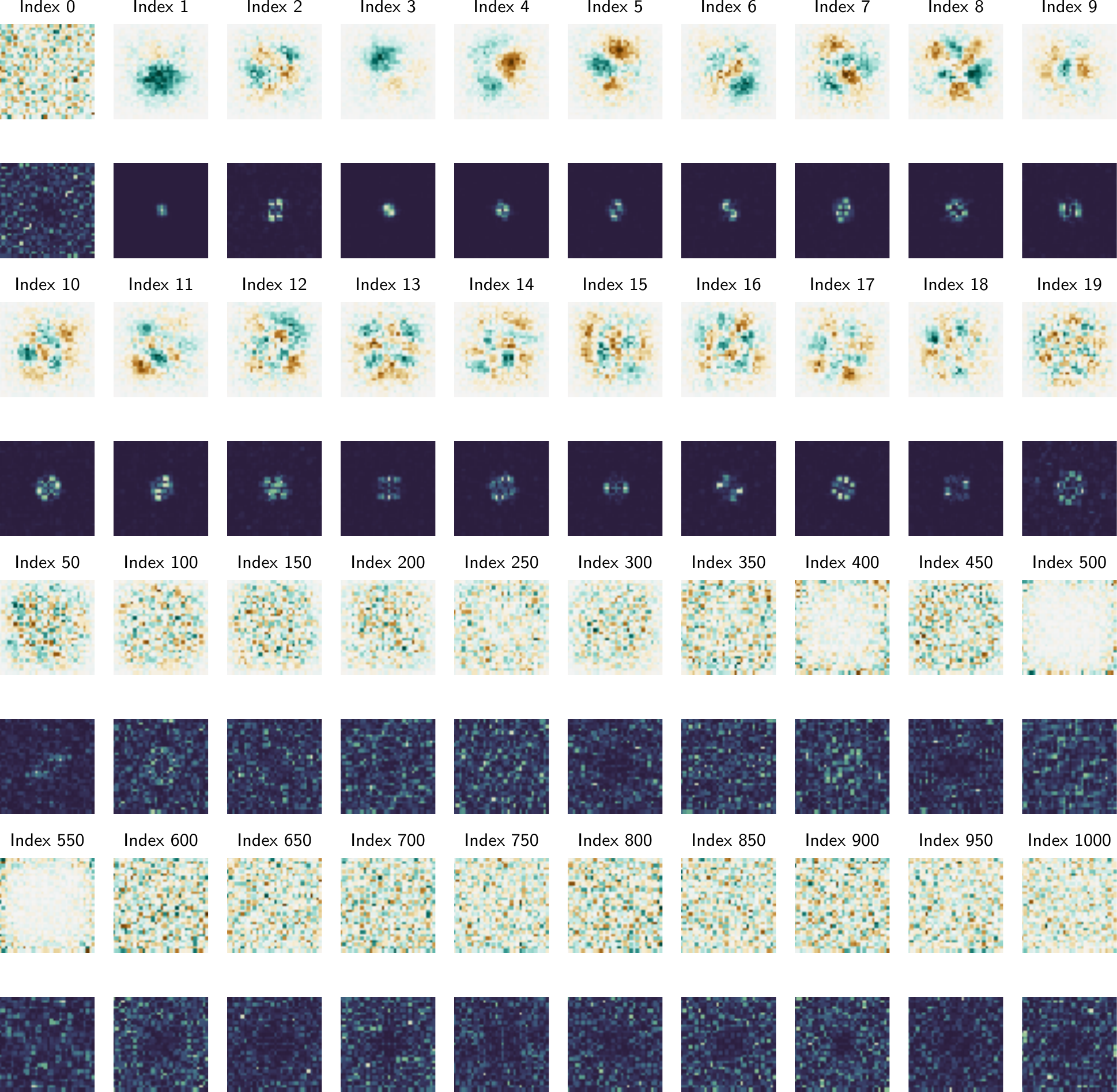}
    \caption{NADs of LeNet obtained through SVD of mixed second derivative}
    \label{fig:lenet_nads_second}
\end{figure}
\newpage

\subsubsection{VGG11}
\begin{figure}[ht!]
    \centering
    \includegraphics[width=\textwidth]{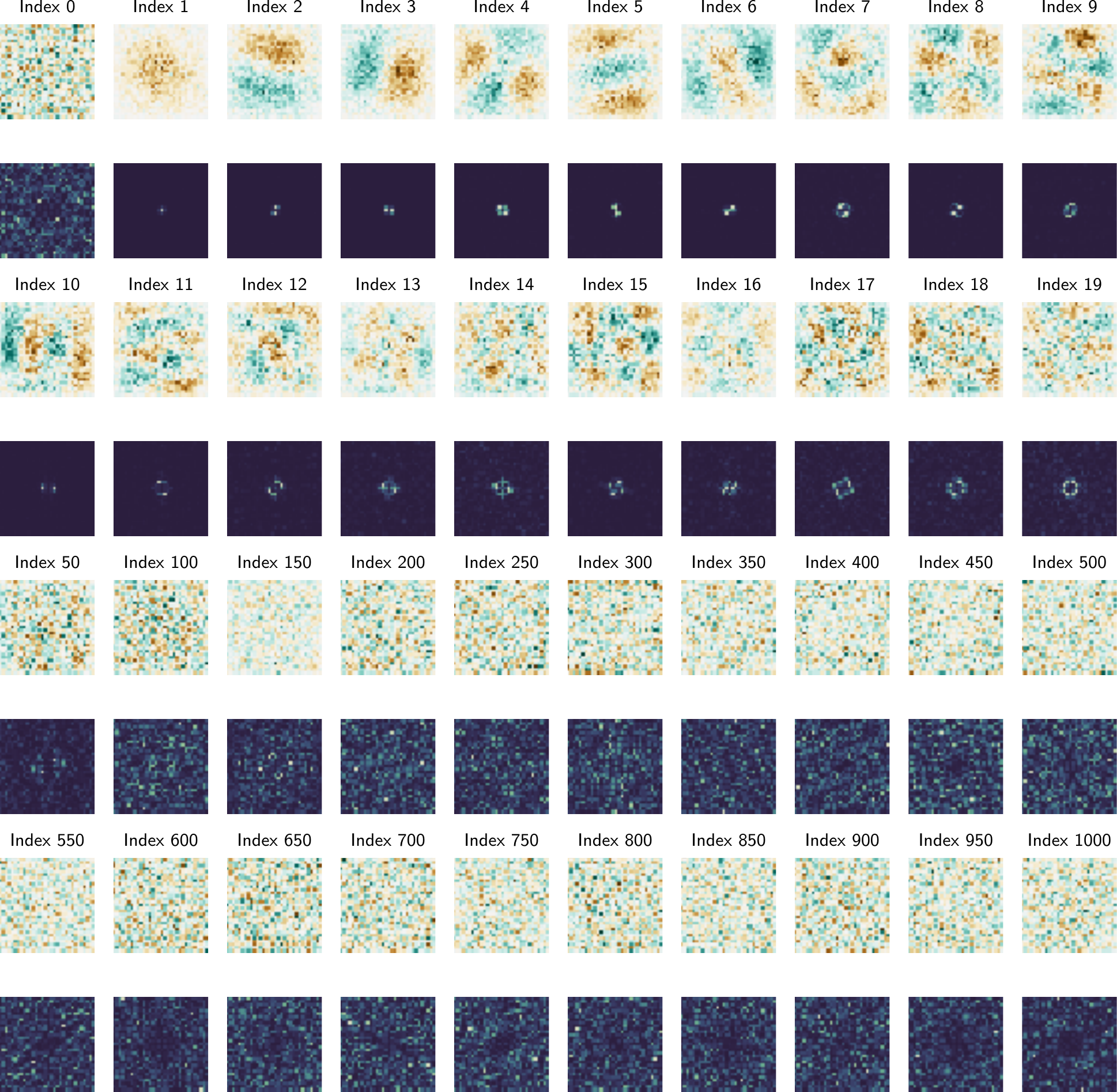}
    \caption{NADs of VGG16 obtained through SVD of mixed second derivative}
    \label{fig:vgg_nads_second}
\end{figure}
\newpage

\subsubsection{ResNet-18}
\begin{figure}[ht!]
    \centering
    \includegraphics[width=\textwidth]{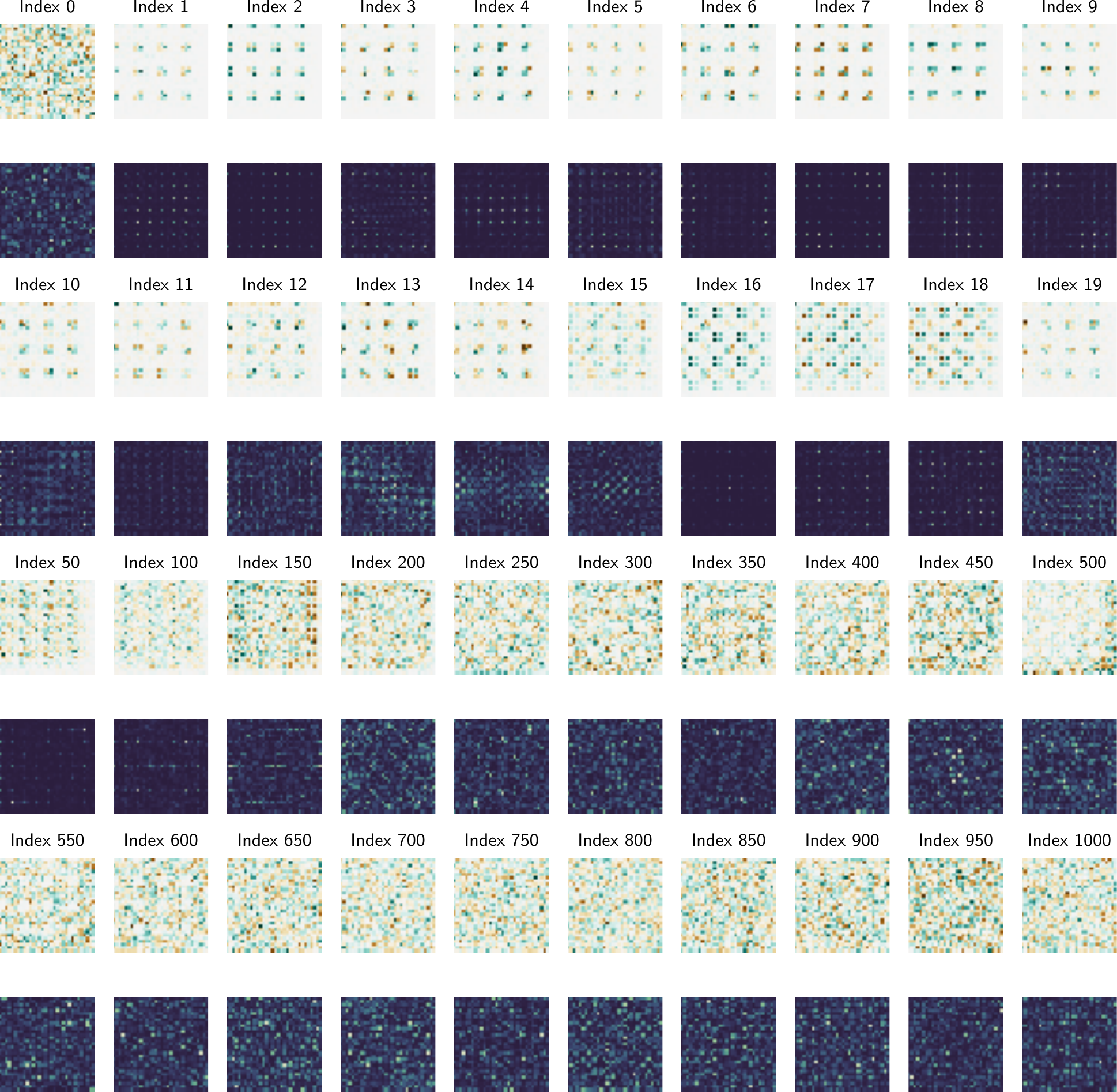}
    \caption{NADs of ResNet-18 obtained through SVD of mixed second derivative}
    \label{fig:resnet_nads_second}
\end{figure}

\subsection{Further experiments with NADs}

We now provide some further experiments using the NADs of some common neural network architectures. First, we give two additional experiments on the performance of a VGG11, and a multilayer perceptron (MLP) with $3$ hidden layers with $500$ neurons each, on a sequence of linearly separable datasets aligned with its NADs. As we can see in Fig~\ref{fig:nad_sup_acc}, the VGG11 behaves very similarly to the other CNNs~(see Fig.~\ref{fig:nad_accs}), only being able to generalize to a few distributions, whereas the MLP can always perfectly generalize to the test distribution. Note also, that the eigenvalue decay on the MLP is much less pronounced. In fact, we believe that this is only a result of the finite set of gradient samples used to perform its eigendecomposition, and we conjecture that in the limit of infinite samples the eigenvalue distribution of the MLP will be completely flat (as we formally proved for the single hidden layer network of Example~\ref{ex:single_hidden_layer}).

\begin{figure}[ht!]
    \centering
    \includegraphics[width=0.7\textwidth]{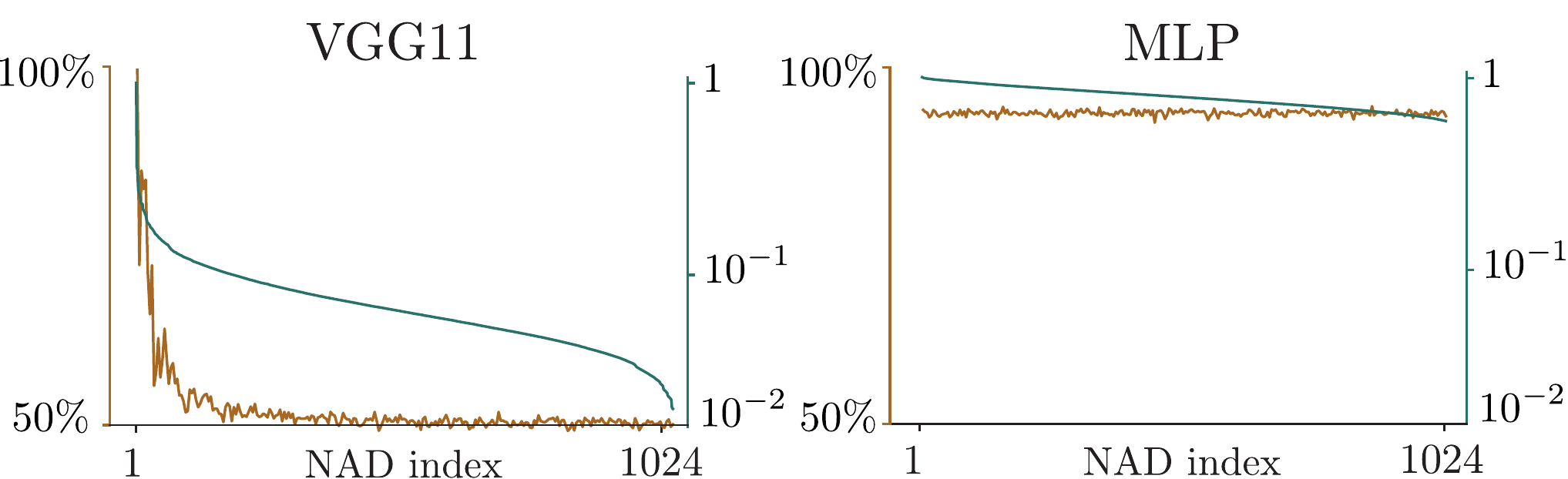}
    \caption{\textbf{(Green)} Normalized covariance eigenvalues and \textbf{(brown)} test accuracies of a MLP and a VGG11 trained on linearly separable distributions parameterized by their NADs. ($\sigma=3$, $\epsilon=1$)}
    \label{fig:nad_sup_acc}
\end{figure}

\subsubsection{Speed of convergence}

NADs also have an effect in optimization. To show this, we tracked the training loss of a LeNet and a ResNet-18 when trained on different $\mathcal{D}(\bm{v})$ parameterized by the NAD sequence. Fig.~\ref{fig:opt_nads} shows these results. As expected, even if in all cases these networks achieved almost a $100\%$ test accuracy, the effect of NADs is clearly visible during optimization. This is, it takes much longer for these networks to converge to small training losses when the discriminative information of the dataset is aligned with the later NADs as opposed to the first ones. This is similar to the phenomenon described in Fig.~\ref{fig:optimization} where we identified the same behaviour with respect to the Fourier basis. However, in that case, higher frequency was not a direct indicator of training hardness (cf.~NAD index).

\begin{figure}[ht!]
\centering
\begin{subfigure}{0.7\textwidth}
    \centering
    \includegraphics[width=\textwidth]{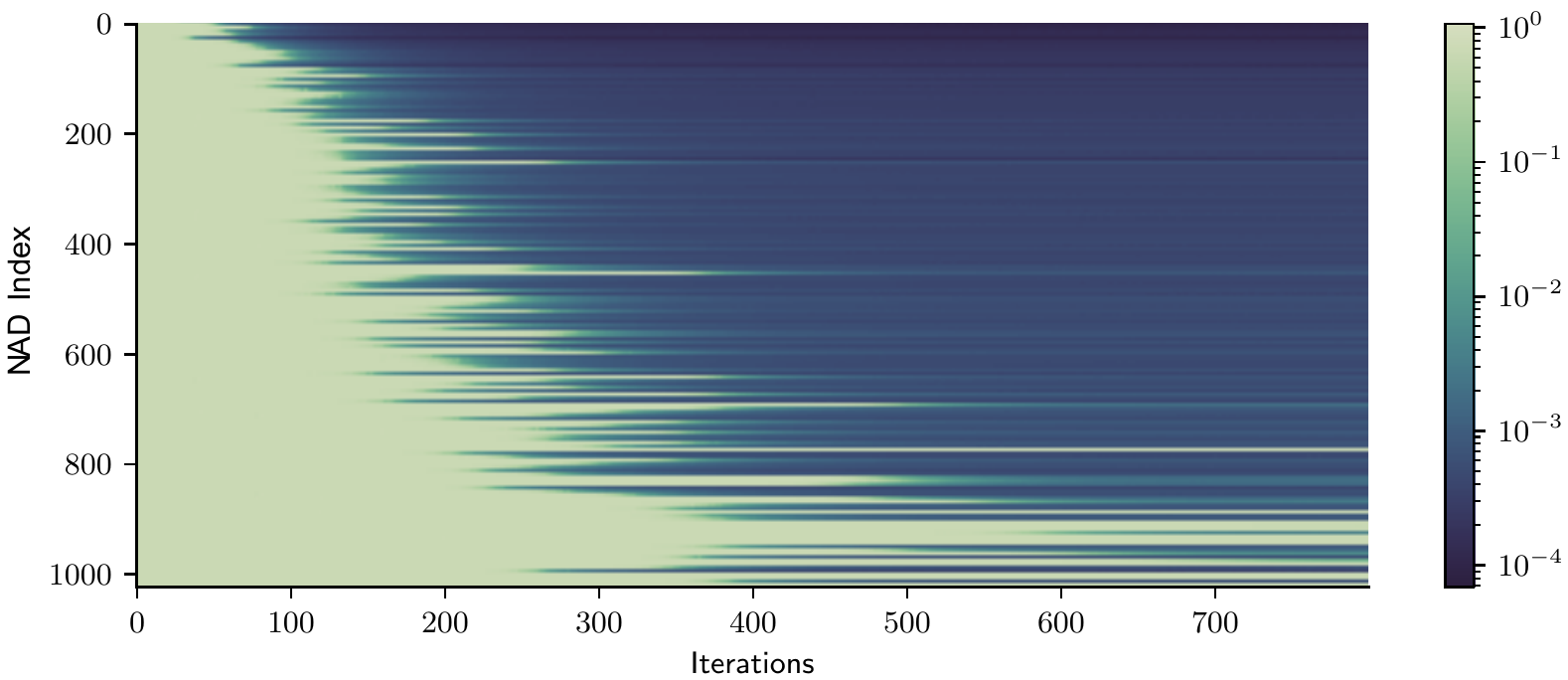}
    \caption{LeNet ($\sigma=0$)}
\end{subfigure}
\hfill
\begin{subfigure}{0.7\textwidth}
    \centering
    \includegraphics[width=\textwidth]{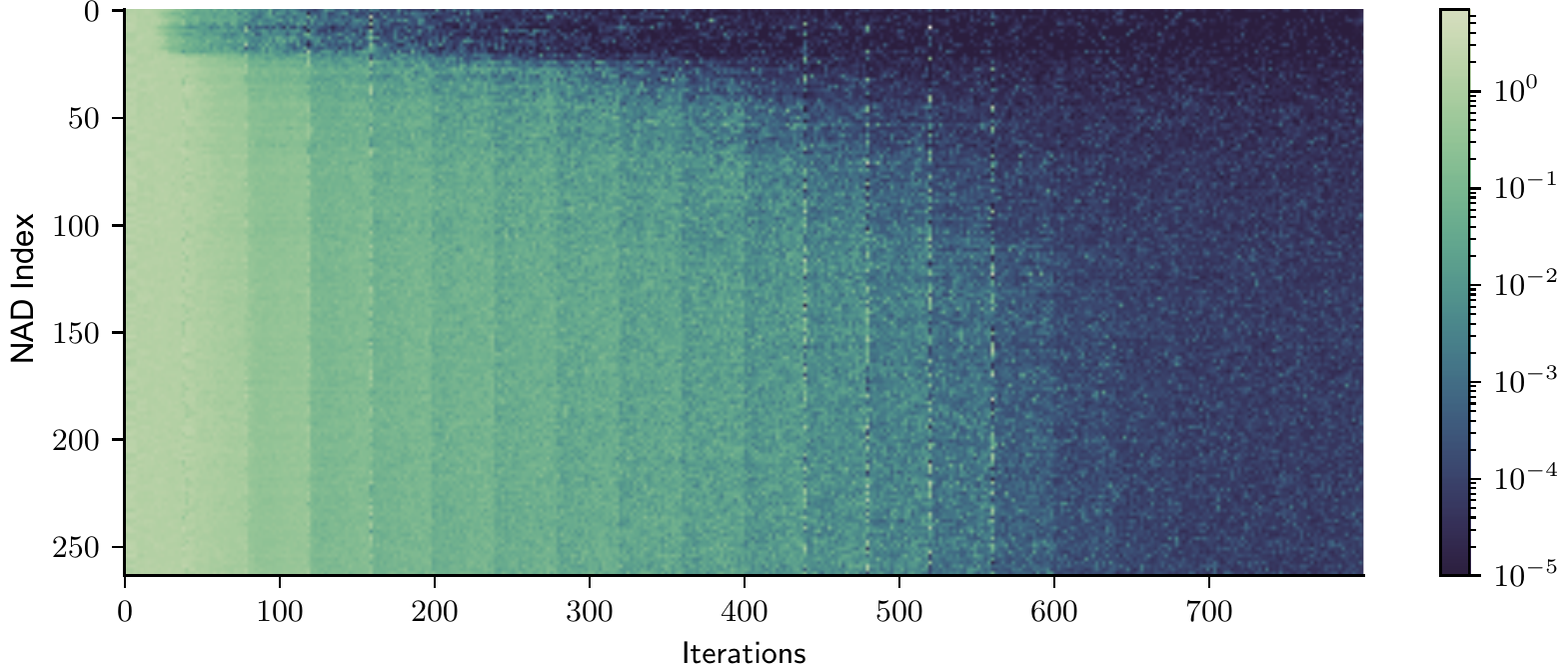}
    \caption{ResNet-18 ($\sigma=1$)}
\end{subfigure}
\caption{Training loss per batch of different networks trained using different training sets drawn from $\mathcal{D}(\bm{v})$ ($\epsilon=1$, and $\sigma$ chosen to accentuate differences). Directions $\bm{v}$ taken from the NAD sequence.}
\label{fig:opt_nads}
\end{figure}

\subsubsection{Generalization vs. number of training samples}
NADs encapsulate the preference of a network to search for discriminative features in some particular directions. This means that a network first tries to fit the training data using features aligned with NADs of lower indices, before proceeding to later ones. In that sense, and for a fixed level of noise $\sigma$, one can argue that, if the discriminative direction of the data is aligned with a NAD of higher index (i.e., not properly aligned with the directional inductive bias of the network), it is quite likely that the network will overfit to some discriminative but non-generalizing solutions, using noisy features that are aligned with NADs of lower indices. In this case, and for reducing such spurious correlations, more training samples might be necessary for the network to ``ignore'' such solutions and seek for other discriminative ones using NADs of higher indices (and hence eventually finding the discriminative and generalizing one).

On the contrary, if the discriminative direction of the data is aligned with a lower NAD index (i.e., properly aligned with the directional inductive bias of the network), then the network tries to fit the training data along the truly generalizing direction earlier; hence, the possibility of overfitting to noisy features appearing along higher NAD indices is reduced. In that sense, even a few training samples might be enough for the network to converge to the generalizing solution.

An illustration of this dependency between the alignment of the generalizing direction with the NADs, and the number of training samples, is shown in Fig.~\ref{fig:gen_vs_samples}. For both cases, it is clear that less training data are required for the network to generalize when the discriminative direction $\bm{v}$ is aligned with the lower NADs of the network. On the contrary, when $\bm{v}$ is aligned with higher NADs, more data is required for the network to ``ignore'' the noisy features and find the generalizing solution. In fact, as clearly observed for the case of ResNet-18, given a large amount of training samples (considering the simplicity of the task) the network can eventually generalize perfectly, regardless the position of the direction $\bm{v}$.

\begin{figure}[h!]
    \begin{center}
    \begin{subfigure}{0.49\textwidth}
        \includegraphics[width=\textwidth]{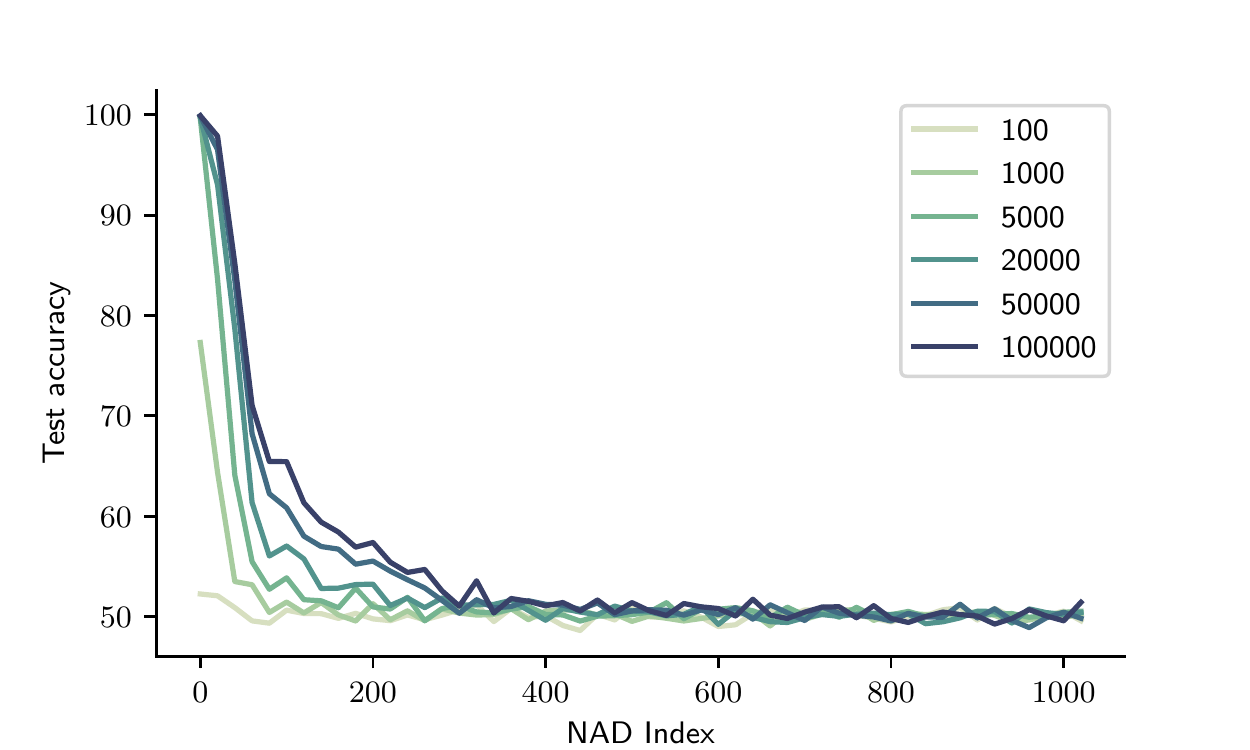}
        \caption{LeNet}
    \end{subfigure}
    \begin{subfigure}{0.49\textwidth}
        \includegraphics[width=\textwidth]{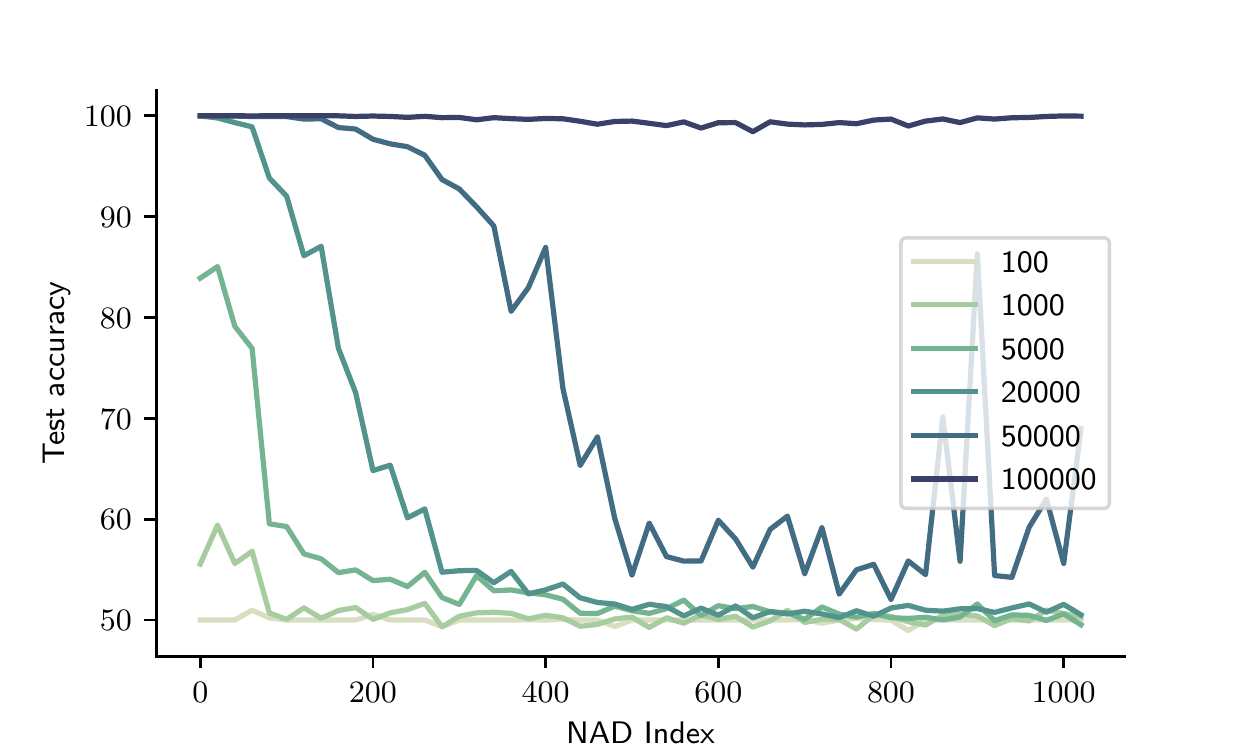}
        \caption{ResNet-18}
    \end{subfigure}
    \caption{Generalization vs number of training samples for two CNNs trained using different training sets drawn from $\mathcal{D}(\bm{v})$ ($\epsilon=1$, and $\sigma=3$). Directions $\bm{v}$ taken from the NAD sequence.}
    \label{fig:gen_vs_samples}
    \end{center}
\end{figure}

\newpage  
\section{Details of experiments on CIFAR10}

All our experiments on CIFAR-10 use networks trained for $50$ epochs using SGD with a linearly decaying learning rate with maximum value $0.21$, fixed momentum $0.9$ and a weight decay of $5\times10^{-4}$. Again, our objective is not to obtain the best achievable performance, but to show relative differences with respect to NADs for a fixed training setup. Hence, the hyperparameters of these networks were not optimized in any way during this work. Instead they were selected from a set of best practices from the DAWNBench submissions that have been empirically shown to give a good trade-off in terms of convergence speed and performance. 

We finish this section with a detailed description of the poisoning experiment. In particular, recall that, in the binary class setting, i.e., $y\in\{-1,+1\}$ an easy way to introduce a poisonous carrier on a sample $\bm{x}$ is to substitute the information on that sample in a given direction by $\epsilon y$. However, this means that, for a given direction $\bm{u}$, we can only allocate at most two classes. A simple extension to the multi-class case, i.e., $y\in\{1,\dots, L\}$, uses therefore $\lceil L / 2\rceil$ directions to poison all samples.

CIFAR-10 has $L=10$ classes, but also, its samples contain information spread along $K=3$ color channels. The NADs that we computed in Sec.~\ref{sec:nads_cnns} were computed for single-channel inputs. Hence, we need to extend them to work in the $K$-channel case. Let $\{\bm{u}_i\}_{i=1}^{D}$ be the NADs of a deep neural network for a single channel input. The NADs of the same architecture with $K$ input channels are $\{\bm{u}_{i}\otimes \bm{e}'_{k}, \; i=1,\dots, D, \; k=1,\dots,K\}$, where $\bm{e}'_k$ represents the $k$\textsuperscript{th} canonical basis vector of $\mathbb{R}^K$.

All in all, using these extensions to the simple setting, we can easily poison CIFAR-10. Given a carrier NAD index $i$, for each sample $\bm{x}_j\in\mathbb{R}^{DK}$ in the training set with associated label $y_j$ we can modify it such that it satisfies $\bm{x}_j^T(\bm{u}_{i}\otimes \bm{e}'_{\lfloor y_j/2\rfloor})=\epsilon (2 \llbracket y_j\rrbracket_2 - 1)$. Note that, for any $\epsilon>0$, this small modification on the training set renders each class linearly separable from the others using only the poisonous features. However a classifier that uses these features will not be able to generalize to the unpoisoned test set.

\end{document}